\documentclass[final,1p,number,times,twocolumn]{elsarticle}

\usepackage{amssymb} \usepackage{graphicx} 
\usepackage{epsfig} 
\usepackage{subfigure}

\usepackage{algorithm} 
\usepackage{algorithmic} 
\usepackage{times}
\usepackage{epsfig} 
\usepackage{amsthm} 
\usepackage{amsmath}
\usepackage{amssymb}

\usepackage{hyperref} 
\usepackage{color}

\usepackage[section] {placeins}

\newcounter{thm_counter} \newcounter{lem_counter}

\newcounter{ass_counter}

\newtheorem{theorem}[thm_counter]{Theorem}
\newtheorem{lemma}[lem_counter]{Lemma}
\newtheorem{assumption}[ass_counter]{Assumption}

\newcommand{\tp}{\tilde{\Phi}} \newcommand{\rca}{\color{red}}
\newcommand{\xlast}{x_{\mbox{\rm\scriptsize last}}}
\newcommand{\chisat}{\chi} 
\def\lassoinf{{\bf LASSO}$\infty$\,} \def\linf{{\bf L}$\infty$\,}
\def\ltwo{{\bf L}2\,} \def\dantzig{{\bf Dantzig}\,}
\def\ltwodantziginf{{\bf L}2{\bf Dantzig}$\infty$\,}

\newcommand{\fmax}{f_{\max}} 
 \def\sjwresolved#1{}
\def\jlresolved#1{}

\journal{Applied and Computational Harmonic Analysis}

\begin{document}

\begin{frontmatter}



\title{Robust Dequantized Compressive Sensing}


\tnotetext[label1]{Corresponding author.}  \address[label2]{Department
  of Computer Sciences, University of Wisconsin$-$Madison, 1210
  W. Dayton St., Madison, WI 53706-1685} \author{Ji
  Liu\corref{label1}\fnref{label2}} \ead{ji-liu@cs.wisc.edu}
\author{Stephen J. Wright\fnref{label2}} \ead{swright@cs.wisc.edu}

\begin{abstract}
We consider the reconstruction problem in compressed sensing in which
the observations are recorded in a finite number of bits. They may
thus contain quantization errors (from being rounded to the nearest
representable value) and saturation errors (from being outside the
range of representable values). Our formulation has an objective of
weighted $\ell_2$-$\ell_1$ type, along with constraints that account
explicitly for quantization and saturation errors, and is solved with
an augmented Lagrangian method. We prove a consistency result for the
recovered solution, stronger than those that have appeared to date in
the literature, showing in particular that asymptotic consistency can
be obtained without oversampling. We present extensive computational
comparisons with formulations proposed previously, and variants
thereof.
\end{abstract}


\end{frontmatter}


\section{Introduction}
\label{introduction}

This paper considers a compressive sensing (CS) system in which the
measurements are represented by a finite number of bits, which we
denote by $B$. By defining a quantization interval $\Delta>0$, and
setting $G:= 2^{B-1} \Delta$, we obtain the following values for
representable measurements:
\begin{equation} \label{eq:qcs.values}
-G+\frac{\Delta}{2}, -G+\frac{3\Delta}{2}, \dotsc,
-\frac{\Delta}{2},\frac{\Delta}{2}, \dotsc, G-\frac{\Delta}{2}.
\end{equation}
We assume in our model that actual measurements are recorded by
rounding to the nearest value in this set.  The recorded observations
thus contain (a) quantization errors, resulting from rounding of the
true observation to the nearest represented number, and (b) saturation
errors, when the true observation lies beyond the range of represented
values, namely, $[-G+\frac{\Delta}{2},G-\frac{\Delta}{2}]$. This setup
is seen in some compressive sensing hardware architectures \citep[see,
  for example,][]{Laska07, Tropp09, Romberg09, Tropp06, Duarte08}.

\sjwresolved{We should establish a consistent callout style for
  equation numbers. I vote to remove ``Eq.'' everywhere. What do you
  think? {\rca I have removed all ``Eq.''.}}

Given a sensing matrix $\Phi \in \mathbb{R}^{M\times N}$ and the
unknown vector $x$, the true observations (without noise) would be
$\Phi x$. We denote the recorded observations by the vector $y \in
\mathbb{R}^M$, whose components take on the values in
\eqref{eq:qcs.values}. We partition $\Phi$ into the following three
submatrices:
\begin{itemize}
\item The saturation parts $\bar{\Phi}_-$ and $\bar{\Phi}_+$, which
  correspond to those recorded measurements that are represented by
  $-G+{\Delta/2}$ or $G-{\Delta/2}$, respectively --- the two extreme
  values in \eqref{eq:qcs.values}. We denote the number of rows in
  these two matrices combined by $\bar{M}$.
\item The unsaturated part $\tp\in \mathbb{R}^{\tilde{M}\times N}$,
  which corresponds to the measurements that are rounded to
  non-extreme representable values.
 \end{itemize}


In some existing analyses \citep{Candes06,JacquesHF11}, the
quantization errors are treated as a random variables following an
i.i.d. uniform distribution in the range $[-{\Delta\over 2},
  {\Delta\over 2}]$. This assumption makes sense in many situations
(for example, image processing, audio/video processing), particularly
when the quantization interval $\Delta$ is tiny.  However, the
assumption of a uniform distribution may not be appropriate when
$\Delta$ is large, or when an inappropriate choice of saturation level
$G$ is made. In this paper, we assume a slightly weaker condition,
namely, that the quantization errors for non-saturated measurements
are independent random variables with zero expectation. (These random
variables are of course bounded uniformly by $\Delta/2$.)

 The state-of-the-art formulation to this
 problem~\citep[see][]{Laska11} is to combine the basis pursuit model
 with saturation constraints, as follows:
\begin{subequations}
\label{eqn_stateart}
\begin{align}
\min_x ~&\|x\|_1 \\
\label{eqn_stateart.L2}
\mbox{s.t.} \quad \|\tp x-\tilde{y}\|^2& \leq \epsilon^2\Delta^2  \;&&\text{($\ell_2$)}\\
\bar{\Phi}_+x &\geq  (G-\Delta)\bold{1}\; &&\text{($+$ saturation)}\\
\bar{\Phi}_-x &\leq  (\Delta-G)\bold{1},\; &&\text{($-$ saturation)}
\end{align}
\end{subequations}
where $\bold{1}$ is a column vector with all entries equal to $1$ and
$\tilde{y}$ is the quantized subvector of the observation vector $y$
that corresponds to the unsaturated measurements.  We refer to this
model as ``\ltwo'' in later discussions. It has been shown that the
estimation error arising from the formulation \eqref{eqn_stateart} is
bounded by $O(\epsilon\Delta)$ in the $\ell_2$ norm
sense~\citep[see][]{Laska11, Candes08, JacquesHF11}.

The paper proposes a robust model that replaces
\eqref{eqn_stateart.L2} with a least-square loss term in the objective
and adds an $\ell_\infty$ constraint:
\begin{subequations}
\label{eqn_ourformulation}
\begin{align}
  \min_{x}~{1\over 2}\|\tp x-\tilde{y}\|^2 &+
  \lambda\Delta\|x\|_1\label{eq_obj}\\
  \mbox{s.t.} \quad \|\tp x-\tilde{y}\|_\infty &\leq \Delta/2 \;&&\text{($\ell_\infty$)} \label{eq_Linfty_cst}\\
\label{eqn_ourformulation.sat+}
 \bar{\Phi}_+x &\geq (G-\Delta)\bold{1} \; &&\text{($+$ saturation)} \\
\label{eqn_ourformulation.sat-} \bar{\Phi}_-x &\leq
(\Delta-G)\bold{1}. \; &&\text{($-$ saturation)}
\end{align}\end{subequations}
We refer to this model as \lassoinf in later discussions.  The
$\ell_\infty$ constraint \eqref{eq_Linfty_cst} arises from the
fact that (unsaturated) quantization errors are bounded by
$\Delta/2$. This constraint may reduce the feasible region for the
recovery problem while retaining feasibility of the true solution
$x^*$, thus promoting more robust signal recovery. From the viewpoint
of optimization, the constraint~\eqref{eqn_stateart.L2} plays the same
role as the least-square loss term in the objective~\eqref{eq_obj},
when the values of $\epsilon$ and $\lambda$ are related
appropriately. However, it will become clear from our analysis that
inclusion of this term in the objective rather than applying the
constraint \eqref{eqn_stateart.L2} can lead a tighter bound on the
reconstruction error.


The analysis in this paper shows that when $\Phi$ is a Gaussian
ensemble, and provided that $S\log N=o(M)$ and several mild
conditions hold, the estimation error of for the solution of
\eqref{eqn_ourformulation} is bounded by
\begin{align*}
 \min \left\{ O \left(\sqrt{S(\log N)/M} \right), O(1) \right\} \Delta,
\end{align*}
with high probability, where $S$ is the sparsity (the number of
nonzero components in $x^*$). This estimate implies that solutions of
\eqref{eqn_ourformulation} are, in the worst case, better than the
state-of-the-art model~\eqref{eqn_stateart}, and also better than the
model in which only the $\ell_\infty$ constraint~\eqref{eq_Linfty_cst}
are applied (in place of the $\ell_2$
constraint~\eqref{eqn_stateart.L2}), as mentioned by
\cite{JacquesHF11}. More importantly, when the number $\tilde{M}$ of
unsaturated measurements goes to infinity faster than $S\log(N)$, the
estimation error for the solution of \eqref{eqn_ourformulation}
vanishes with high probability. (The model \eqref{eqn_stateart} does
not indicate such an improvement when more measurements are
available.) Although
\citet{JacquesHF11} show that the estimation error can be eliminated
only using an $\ell_p$ constraint (in place of the $\ell_2$
constraint~\eqref{eqn_stateart.L2}) when $p\rightarrow \infty$, the
oversampling condition (that is, the number of observations required)
is more demanding than for our formulation \eqref{eqn_ourformulation}.

We use the alternating direction method of multipliers (ADMM)
\citep[see][]{Eckstein92,Boyd11} to solve \eqref{eqn_ourformulation}.
 The computational results reported in Section~\ref{simulation}
 compare the solution properties for \eqref{eqn_ourformulation} to
 those for \eqref{eqn_stateart} and other formulations. In some of our
 examples, we consider choices for the parameter $\lambda$ and
 $\epsilon$ that admit the true solution $x^*$ as a feasible point
 with a specified level of confidence. We find that for these choices
 of $\lambda$ and $\epsilon$, the model~\eqref{eqn_ourformulation}
 yields more accurate solutions than the alternatives, where the
 signal is sparse and high confidence is desired.



\subsection{Related Work}

There have been several recent works on CS with quantization and
saturation.  \citet{Laska11} propose the formulation
\eqref{eqn_stateart}. \citet{JacquesHF11} replace the $\ell_2$
constraint \eqref{eqn_stateart.L2} by an $\ell_p$ constraint ($2\leq
p< \infty$) to handle the oversampling case, and show that values $p$
greater than $2$ lead to an improvement of factor $1/\sqrt{p+1}$ on
the bound of error in the recovered signal. The model of
\citet{Zymnis10} allows Gaussian noise in the measurements before
quantization, and solves the resulting formulation with an
$\ell_1$-regularized maximum likelihood formulation.\sjwresolved{This
  sentence is not quite clear to me. Their model is
  $y=Quantize[X\beta+Gaussian Noise]$} The average distortion
introduced by scalar, vector, and entropy coded quantization of CS is
studied by \citet{Dai11}.

The extreme case of 1-bit CS (in which only the sign of the
observation is recorded) has been studied by \citet{Gupta10} and
\citet{Boufounos08}.  In the latter paper, the $\ell_1$ norm objective
is minimized on the unit ball, with a sign consistency constraint. The
former paper proposes two algorithms that require at most $O(S\log N)$
measurements to recover the unknown support of the true signal (though
they cannot recover the magnitudes of the nonzeros reliably).

\subsection{Notation}

We use $\| \cdot \|_p$ to denote the $\ell_p$ norm, where $1\leq p\leq
\infty$, with $\| \cdot \|$ denoting the $\ell_2$ norm. We use $x^*$
for the true signal, $\hat{x}$ as the estimated signal (the solution
of \eqref{eqn_ourformulation}), and $h=\hat{x}-x^*$ as the
difference. As mentioned above, $S$ denotes the number of nonzero
elements of $x^*$.

For any $z \in \mathbb{R}^N$, we use $z_i$ to denote the $i$th
component and $z_T$ to denote the subvector corresponding to index set
$T\subset \{1,2,...,N\}$. Similarly, we use $\tp_{T}$ to denote the
column submatrix of $\tp$ consisting of the columns indexed by $T$.
The cardinality of $T$ is denoted by $|T|$. We use $\tp_{i}$ to denote
the $i$th column of $\tp$.

In discussing the dimensions of the problem and how they are related
to each other in the limit (as $N$ and $\tilde{M}$ both approach
$\infty$), we make use of order notation. If $\alpha$ and $\beta$
are both positive quantities that depend on the dimensions, we write
$\alpha = O(\beta)$ if $\alpha$ can be bounded by a fixed multiple
of $\beta$ for all sufficiently large dimensions. We write $\alpha =
o(\beta)$ if for {\em any} positive constant $\phi>0$, we have
$\alpha \le \phi \beta$ for all sufficiently large dimensions. We
write $\alpha = \Omega(\beta)$ if both $\alpha=O(\beta)$ and
$\beta=O(\alpha)$.

The projection onto the $\ell_\infty$ norm ball with the radius
$\lambda$ is
\[
\mathcal{P}_\infty(x,\lambda) := \mbox{sign}(x)\odot\min(|x|,
\lambda)
\]
where $\odot$ denotes componentwise multiplication and
$\mbox{sign}(x)$ is the sign vector of $x$. (The $i$th entry of
$\mbox{sign}(x)$ is $1$, $-1$, or $0$ depending on whether $x_i$ is
positive, negative, or zero, respectively.)

The indicator function $\mathbb{I}_{\Pi}(\cdot)$ for a set $\Pi$ is
defined to be $0$ on $\Pi$ and $\infty$ otherwise.

We partition the sensing matrix $\Phi$ according to saturated and
unsaturated measurements as follows:
\begin{equation} \label{eq:def.phis}
\bar{\Phi}=\left[\begin{matrix} - \bar{\Phi}_- \\
\bar{\Phi}_+
\end{matrix}\right]~\text{and}~\Phi=\left[\begin{matrix} \tilde{\Phi} \\
\bar{\Phi}
\end{matrix}\right].
\end{equation}
The maximum column norm in $\tp$ is denoted by $\fmax$, that is,
\begin{equation} \label{eq:fmax}
\fmax:=\max_{i\in
  \{1,2,\dotsc,N\}}\|\tp_{i}\|.
\end{equation}
We define the following quantities associated with a matrix $\Psi$
with $N$ columns:
\begin{subequations}
\begin{align} \label{eq:defrho-}
\rho^-(k, \Psi) & := \min_{|T|\leq k, h\in
  \mathbb{R}^{N}}{\|\Psi_T h_T\|^2 \over \|h_T\|^2} \\
\label{eq:defrho+}
\rho^+(k, \Psi) &:= \max_{|T|\leq k, h\in \mathbb{R}^{N}}{\|\Psi_{T}
h_T\|^2 \over \|h_T\|^2}.
\end{align}
\end{subequations}
We use the following abbrevations in some places:
\begin{align*}
\rho^-(k):=\rho^-(k, \Phi), \quad &
{\rho}^+(k):={\rho}^+(k, \Phi), \\
\tilde{\rho}^-(k):={\rho}^-(k, \tilde{\Phi}), \quad &
\tilde{\rho}^+(k):={\rho}^+(k, \tilde{\Phi}), \\
\bar{\rho}^-(k):={\rho}^-(k, \bar{\Phi}), \quad &
\bar{\rho}^+(k):={\rho}^+(k, \bar{\Phi}).
\end{align*}
%


Finally, we denote $(z)+ := \max\{z, 0\}$.

\subsection{Organization}

The ADMM optimization framework for solving \eqref{eqn_ourformulation}
is discussed in Section~\ref{algorithm}.  Section~\ref{theory}
analyzes the properties of the solution of \eqref{eqn_ourformulation}
in the worst case and compares with existing results. Numerical
simulations and comparisons of various formulations are reported in
Section~\ref{simulation} and some conclusions are offered in
Section~\ref{conclusion}. Proofs of the claims in Section~\ref{theory}
appear in the appendix.

\section{Algorithm}
\label{algorithm}

This section describes the ADMM algorithm for solving
\eqref{eqn_ourformulation}. For simpler notation, we combine the
saturation constraints as follows:
\[
\left[ \begin{matrix} - \bar{\Phi}_- \\ \bar{\Phi}_+ \end{matrix}
\right]  x \ge \left[ \begin{matrix} (G-\Delta) \bold{1} \\
(G-\Delta) \bold{1} \end{matrix} \right] \;\; \Leftrightarrow \;\;
\bar{\Phi} x \ge \bar{y},
\]
where $\bar{\Phi}$ is defined in \eqref{eq:def.phis} and $\bar{y}$ is
defined in an obvious way. To specify ADMM, we introduce auxiliary
variables $u$ and $v$, and write \eqref{eqn_ourformulation} as
follows.
\begin{equation}
\begin{aligned}
\min_x~&{1\over 2}\|\tp x-\tilde{y}\|^2+\lambda\|x\|_1\\
\mbox{s.t.} \quad u &=\tilde{\Phi}x-\tilde{y}\\
v &= \bar{\Phi}x-\bar{y}\\
\|u\|_\infty &\leq \Delta/2\\
 v &\geq \bold{0}. \label{eqn_L2inftyADMM}
\end{aligned}
\end{equation}
Introducing Lagrange multipliers $\alpha$ and $\beta$ for the two
equality constraints in \eqref{eqn_L2inftyADMM}, we write the
augmented Lagrangian for this formulation, with prox parameter
$\theta>0$ as follows:
\begin{align*}
L_{A}(x, u, v, \alpha, \beta)
& = {1\over 2}\|\tp x-\tilde{y}\|^2+\lambda\|x\|_1+ \langle \alpha, u-\tilde{\Phi}x+\tilde{y} \rangle + \langle \beta, v-\bar{\Phi}x+\bar{y} \rangle \\
& \quad + {\theta\over 2}\|u-\tilde{\Phi}x+\tilde{y}\|^2 + {\theta\over
2}\|v-\bar{\Phi}x +\bar{y} \|^2 + \mathbb{I}_{\|u\|_\infty\leq
\Delta/2}(u) + \mathbb{I}_{v\geq 0}(v)
\end{align*}
At each iteration of ADMM, we optimize this function with respect to
the primal variables $u$ and $v$ in turn, then update the dual
variables $\alpha$ and $\beta$ in a manner similar to gradient
descent.  The penalty parameter $\theta$ may be increased before
proceeding to the next iteration.

We summarize the ADMM algorithm in {\bf Algorithm~\ref{alg:ADMM}}.
\begin{algorithm}[h]
\caption{ADMM for~\eqref{eqn_L2inftyADMM}} \label{alg:ADMM}
\begin{algorithmic}[1]
\REQUIRE $\tilde{\Phi}$, $\tilde{y}$, $\bar{\Phi}$, $\bar{y}$, $\Delta$, $K$, and $x$; \\
\STATE Initialize $\theta > 0$, $\alpha=0$, $\beta=0$,
$u=\tilde{\Phi}x-\tilde{y}$, and
$v=\bar{\Phi}x-\bar{y}$; \\
\FOR{$k=0:K$} \STATE $u \leftarrow \arg\min_{u} \, L_{A}(x, u, v,
\alpha, \beta)$, that is,  $u \leftarrow \mathcal{P}_{\infty}(\tilde{\Phi}x-\tilde{y}-\alpha/\theta,
\Delta/2)$; \\
\STATE $v \leftarrow \arg\min_{v} \, L_{A}(x, u, v, \alpha, \beta)$, that is,  $v \leftarrow \max(\bar{\Phi}x-\bar{y}-\beta/\theta, 0)$; \\
\STATE\label{eqn_alg_lasso} $x \leftarrow \arg\min_{x} \, L_{A}(x, u, v, \alpha, \beta)$; \\
\STATE $\alpha \leftarrow \alpha + \theta(u-\tilde{\Phi}x+\tilde{y})$; \\
\STATE $\beta \leftarrow \beta+\theta(v-\bar{\Phi}x+\bar{y})$; \\
\STATE Possibly increase $\theta$; 
\\
\IF{stopping criteria is satisfied} \STATE break; \ENDIF \ENDFOR
\end{algorithmic}
\end{algorithm}

The updates in Steps~3 and 4 have closed-form solutions, as shown.
The function to be minimized in Step~5
consists of an $\|x\|_1$ term in conjunction with a quadratic term in
$x$. Many algorithms can be applied to solve this problem, e.g., the
SpaRSA algorithm \citep{Wright09}, the accelerated first order method
\citep{Nesterov07}, and the FISTA algorithm \citep{Beck09}. The update
strategy for $\theta$ in Step~7 is flexible. We use the following
simple and useful scheme from \citet{He00} and \citet{Boyd11}:
\begin{equation}
\theta:=  \left\{
   \begin{aligned}
   &\theta\tau&&\text{if $\|r\|> \mu\|d\|$}  \\
   &\theta/\tau&&\text{if $\|r\|< \mu\|d\|$} \\
   &\theta&&\text{otherwise},
   \end{aligned}
   \right.
\end{equation}
where $r$ and $d$ denote the primal and dual residual errors
respectively, specifically,
\begin{equation*}
r= \left[\begin{matrix}
         u-\tp x+\tilde{y}\\
        v-\bar{\Phi}v +\bar{y}
        \end{matrix}\right]~\text{and}~d=\theta \left[\begin{matrix}
         \tp (x-\xlast)\\
        \bar{\Phi}(x-\xlast)
        \end{matrix}\right],
\end{equation*}
where $\xlast$ denotes the previous value of $x$. The parameters $\mu$
and $\tau$ should be greater than $1$; we used $\mu=10$ and
$\tau=2$. Convergence results for ADMM can be found in \citep{Boyd11},
for example.


\section{Analysis}
\label{theory}

The section analyzes the properties of the solution obtained from
our formulation \eqref{eqn_ourformulation}. In
Subsection~\ref{sec:theory.boundh}, we obtain bounds on the norm of
the difference $h$ between the estimator $\hat{x}$ given by
\eqref{eqn_ourformulation} and the true signal $x^*$. Our bounds
require the true solution $x^*$ to be feasible for the formulation
\eqref{eqn_ourformulation}; we derive conditions that guarantee that
this condition holds, with a specified probability. In
Subsection~\ref{sec:theory.boundm}, we estimate the constants that
appear in our bounds under certain assumptions, including an
assumption that the full sensing matrix $\Phi$ is Gaussian.

We formalize our assumption about quantization errors as follows.
\begin{assumption} \label{ass:1}
The quantization errors $(\tp x^* - \tilde{y})_i$,
$i=1,2,\dotsc,\tilde{M}$ are independently distributed with expectation
$0$.
\end{assumption}

\noindent 
(Note that since $\tilde{\Phi}$ and $\tilde{y}$ refer to the
unsaturated data, the quantization error are bounded uniformly by
$\Delta/2$.)

\subsection{Estimation Error Bounds} \label{sec:theory.boundh}

The following error estimate is our main theorem, proved in the
appendix.  \sjwresolved{I think we have to put the statement of
  Lemma~\ref{lem_left.LASSO} here, to define the notation $A_0$ and
  $A_1$. I think this might be preferable to a forward
  reference. {\rca I move the definitions of $A_0(\Psi)$ and
    $A_1(\Psi)$ into Theorem 1. What do you think?}}

\begin{theorem} \label{thm_main}
Assume that the true signal $x^*$ satisfies
\begin{equation}\label{eqn_feasible}
\|\tp^T(\tp x^* - \tilde{y})\|_\infty \leq \lambda\Delta /2,
\end{equation}
for some value of $\lambda$. Let $s$ be a positive integer in
  the range $1,2,\dotsc,N$, and define
\begin{subequations}
\begin{align}
\label{eq:def.A0}
\bar{A}_0(\Psi):=&{\rho}^-(2s, \Psi)-3[{\rho}^+(3s, \Psi)-{\rho}^-(3s, \Psi)]\\
\label{eq:def.A1}
\bar{A}_1(\Psi):=&4[{\rho}^+(3s, \Psi)-{\rho}^-(3s,\Psi)],\\
\bar{C}_1(\Psi):=&4+{\sqrt{10}A_1(\Psi)}/{A_0(\Psi)},\label{eq:def.C1}\\
\bar{C}_2(\Psi):=&\sqrt{10 /A_0(\Psi)} \label{eq:def.C2}.
\end{align}
\end{subequations}
\sjwresolved{I don't think $B_S$ appears in the analysis any more.
Can
  we delete it? {\rca Yes, you are right. It should not appear in the
    theorem. Actually, I found a mistake about the bound caused by the
    saturation constraint in this morning. I deleted the inequality
    about it, but have no time to check if there is something
    inconsistent.} Have you checked and is it OK now? {\rca Yes, I have deleted all $B_S$ in this draft.}}
We have that for any $T_0\subset \{1,2,...,N\}$ with $s=|T_0|$, if
$A_0(\tp)>0$, then
\begin{subequations}
\begin{align}
\|h\| \leq& 2\bar{C}_2(\tp)^2\sqrt{s}\lambda\Delta + \left[\bar{C}_1(\tp)/\sqrt{s}\right] \|x^*_{T_0^c}\|_1 +
2.5\bar{C}_2(\tp)\sqrt{\lambda\Delta \|x^*_{T_0^c}\|_1},
\label{eqn_thm:LASSO.1} \\
\|h\|\leq & \bar{C}_2(\tp)\sqrt{\tilde{M}}\Delta + \left[\bar{C}_1(\tp)/\sqrt{s}\right] \|x^*_{T_0^c}\|_1.
\label{eqn_thm:LASSO.2}
\end{align}
\end{subequations}

Suppose that Assumption~\ref{ass:1} holds, and let $\pi \in (0,1)$
be given. If we define $\lambda = \sqrt{2\log {2N/\pi}}\fmax$
in~\eqref{eqn_ourformulation}, then with probability at least
$P=1-\pi$, the inequalities \eqref{eqn_thm:LASSO.1} and
\eqref{eqn_thm:LASSO.2} hold.
\end{theorem}
From the proof in the appendix, one can see that the estimation error
bound ~\eqref{eqn_thm:LASSO.1} is mainly determined by the
least-squares term in the objective~\eqref{eq_obj}, whereas the
estimation error bound~\eqref{eqn_thm:LASSO.2} arises from the
$L_\infty$ constraint~\eqref{eq_Linfty_cst}.
\sjwresolved{I could not make sense of this last phrase. {\rca
Sorry, I should remove this argument.} I commented out the last phrase -
  OK? {\rca I removed the last phrase in the early version. It is
    about $B_S$ which does not appear in our new version any more.}
  Ji, I am a bit confused. Could you take a look at the phrase I
  commented out here. It mentions the ``saturation constraints''
  appearing in the bound \eqref{eqn_thm:LASSO.2}, but in fact they do
  not appear there (they appeared in the old bound). So I think this
  phrase also related to material which has since been removed. Do you
  agree that the text is now OK as it stands? {\rca Yes, I think the
    current text is OK, since we do not discuss anything about the
    saturation constraints in the following text.}}

If we take $T_0$ as the support set of $x^*$, only the first terms
in~\eqref{eqn_thm:LASSO.1} and \eqref{eqn_thm:LASSO.2} remain. 

The condition $A_0(\tp)>0$ is a sort of restricted isometry (RIP)
condition required in~\citep{Laska11}--- it assumes reasonable conditioning of column submatrices
of $\tp$ with $O(S)$ columns. Specifically, the number of measurements
$\tilde{M}$ required to satisfy $\bar{A}_0(\tp)>0$ and RIP are of the
same order: $O(S\log (N))$.

\subsection{Estimating the Constants} \label{sec:theory.boundm}

Here we discuss the effect of the least-squares term and the
$\ell_\infty$ constraints by comparing the leading terms on the
right-hand sides of~\eqref{eqn_thm:LASSO.1} and
\eqref{eqn_thm:LASSO.2}.  To simplify the comparison, we make the
following assumptions.
\begin{itemize}
\item[(i)] $\Phi$ is a Gaussian random matrix, that is, each entry is
  i.i.d., drawn from a standard Gaussian distribution
  $\mathcal{N}(0,1)$.
\item[(ii)] the confidence level $P=1-\pi$ is fixed.
\item[(iii)] $s$ is equal to the sparsity number $S$.
\item[(iv)] $S\log N = o(M)$.
\item[(v)] the saturation ratio $\chisat := \bar{M}/M$ is smaller than
  a small positive threshold that is defined in
  Theorem~\ref{thm_rhobound}.
  \item[(vi)] $T_0$ is taken as the support set of $x^*$, so that
    $x^*_{T_0^c}=0$.
\end{itemize}
Note that (iii) and (iv) together imply that $s=S \ll M$, while
(v) implies that $\tilde{M} = \Omega(M)$.

The discussion following Theorem~\ref{thm_rhobound} in Appendix
indicates that under these assumptions, the quantities defined in \eqref{eq:def.C1}, \eqref{eq:def.C1}, and
\eqref{eq:fmax} satisfy the following estimates:
\[
\bar{C}_1(\tp)= \Omega(1), \quad \bar{C}_2(\tp) =
\Omega(1/\sqrt{M}),\quad \fmax = \Omega(\sqrt{M}),
\]
with high probability, for sufficiently high dimensions. Using the
estimates in Theorem~\ref{thm_rhobound}, with the setting of $\lambda$
from Theorem~\ref{thm_main}, we have
\begin{subequations}
\begin{align}
\bar{C}_2(\tp)^2 \sqrt{s}\lambda\Delta &= O \left( {\sqrt{S\log N}\fmax\Delta \over M} \right)
=
O \left(\sqrt{S\log N \over M}\Delta \right)\rightarrow 0, \label{eq:ls_bound_part}\\
\bar{C}_2(\tp)\sqrt{\tilde{M}}\Delta &= O \left( {\sqrt{\tilde{M}}\Delta\over \sqrt{M}}
\right) = O\left(\Delta\right).\label{eq:linfty_bound_part}
\end{align}
\end{subequations}
By combining the estimation error bounds \eqref{eqn_thm:LASSO.1} and
\eqref{eqn_thm:LASSO.2}, we have
\begin{align}
  \|h\| \leq \min \, \left\{O \left(\sqrt{S(\log N)/M} \right),
  O(1) \right\} \Delta.\label{eqn_boundsimple}
\end{align}
In the regime described by assumption (iv), \eqref{eq:ls_bound_part}
will be asymptotically smaller than \eqref{eq:linfty_bound_part}. The
bound in~\eqref{eqn_boundsimple} has size $O \left(\sqrt{S(\log N)/
  M}\Delta\right)$, consistent with the upper bound of the Dantzig
selector \citep{Candes07a} and LASSO \citep{Zhang09a}\footnote{Their
  bound is $O \left( \sqrt{S(\log N)/ M}\sigma \right)$ where
  $\sigma^2$ is the variance of the observation noise which, in the
  classical setting for the Dantzig selector and LASSO, is assumed to
  follow a Gaussian distribution.}.  Recall that the estimation error
of the formulation \eqref{eqn_stateart} is $O \left(\|\tp
x^*-\tilde{y}\|/\sqrt{\tilde{M}} \right)$ \citep{JacquesHF11, Laska11}
under the RIP condition, for the number of measurements defined in
(iv). Since $\|\tp x^*-\tilde{y}\|=O \left( \sqrt{\tilde{M}}\Delta
\right)$ \citep{JacquesHF11}, this estimate is consistent with the
error that would be obtained if we imposed only the $\ell_{\infty}$
constraint \eqref{eq_Linfty_cst} in our formulation.  Note that it
does not converge to zero even all assumptions (i)-(vi) hold.
Under the assumption (iv), the estimation error for
\eqref{eqn_ourformulation} will vanish as the dimensions grow, with
probability at least $1-\pi$.
By contrast, \citet{JacquesHF11} do not account for saturation in
their formulation and show that the estimation error converges to $0$
using an $\ell_p$ constraint in place of \eqref{eqn_stateart.L2} when
$p \to \infty$ and oversampling happens --- specifically, $M\geq
\Omega\left(\left(S\log (N/S)\right)^{p/2}\right)$. Weaker
oversampling conditions are available using our formulation
\eqref{eqn_ourformulation}. For example, ${M} = S (\log N)^2$ would
produce consistency in our formulation, but not in
\eqref{eqn_stateart}.



\section{Simulations} \label{simulation} 

This section compares results for five variant formulations. The first
one is our formulation \eqref{eqn_ourformulation}, which we refer to
as \lassoinf. We also tried a variant in which the $\ell_\infty$
constraint~\eqref{eq_Linfty_cst} was omitted from
\eqref{eqn_ourformulation}. The recovery performance for this variant
was uniformly worse than for \lassoinf, so we do not show it in our
figures. (It is, however, sometimes better than the formulations
described below, and uniformly better than \dantzig.)
The remaining four alternatives are based on the following model, in
which the $\ell_2$ norm of the residual appears in a constraint
(rather than in the objective) and a constraint of Dantzig type also
appears:
\begin{subequations}
\label{eqn_full}
\begin{alignat}{2}
\min_x~&\|x\|_1 && \\
\label{eqn_full.L2}
\mbox{s.t.} \quad \|\tp x-\tilde{y}\|^2 &\leq \epsilon^2\Delta^2 \; &&\text{($\ell_2$)}\\
\label{eqn_full.Linf}
\|\tp x-\tilde{y}\|_\infty &\leq \Delta/2 \; &&\text{($\ell_\infty$)}\\
\label{eqn_full.Dantzig}
\|\tp^T(\tp x-\tilde{y})\|_\infty &\leq \lambda\Delta/2 \; &&\text{(Dantzig)}\\
\label{eqn_full.sat+}
\bar{\Phi}_+x &\geq (G-\Delta)\bold{1} \; &&\text{($+$ saturation)} \\
\label{eqn_full.sat-} \bar{\Phi}_-x &\leq (\Delta-G)\bold{1}. \;
&&\text{($-$ saturation)}
\end{alignat}
\end{subequations}
\sjwresolved{There is a problem with $\lambda$, which is used in a
  different sense in \eqref{eqn_full} and
  \eqref{eqn_ourformulation}. Hence the comments in the result of this
  section, about choice of $\lambda$ are very confusing. Which version
  of $\lambda$ do you refer to in these discussions? Do you need a
  different item of notation? Do more discussions need to be added for
  the LASSO formulation? It seems to me that the discussion on pages
  7-8 will need to be modified , if we are going to keep reporting
  results from all five models. {\rca Actually, the $\lambda$ in our
    formulation and the Dantzig constraint are consistent, since both
    of them are used to bound $\|\tp^T(\tp x-\tilde{y})\|_\infty \leq
    \lambda\Delta/2$.} OK.}
The four formulations are obtained from this model as
follows.\sjwresolved{In these descriptions you don't say whether the
  saturation constraints are enforced. Please be explicit about it. In
  particular, in \linf, are the saturation constraints enforced?  We
  said earlier that \citet{JacquesHF11} does {\em not} enforce
  saturation - so which is correct? {\rca Actually, we enforce the
    saturation constraint for all algorithms, since the saturation
    constraint is always helpful for all
    algorithms. \citet{JacquesHF11} did not assume the saturation
    noise, but we still contain the saturation constraint for the
    $L_\infty$ model in our comparison (Otherwise it is unfair for
    comparison.)}}
\begin{itemize}
\item \linf: an $\ell_{\infty}$ constraint model that enforces
  \eqref{eqn_full.Linf}, \eqref{eqn_full.sat+}, and
  \eqref{eqn_full.sat-}, but not \eqref{eqn_full.L2} or
  \eqref{eqn_full.Dantzig}. This model is obtained by letting $p \to
  \infty$ in \citet{JacquesHF11} and adding saturation constraints.
\item \ltwo: an $\ell_2$ constraint model (that is, the
  state-of-the-art model \eqref{eqn_stateart} \citep{Laska11}) that
  enforces \eqref{eqn_full.L2}, \eqref{eqn_full.sat+}, and
  \eqref{eqn_full.sat-}, but not \eqref{eqn_full.Linf} or
  \eqref{eqn_full.Dantzig};
\item \dantzig: the Dantzig constraint algorithm with saturation
  constraints, which enforces \eqref{eqn_full.Dantzig},
  \eqref{eqn_full.sat+}, and \eqref{eqn_full.sat-} but not
  \eqref{eqn_full.L2} or \eqref{eqn_full.Linf};
\item \ltwodantziginf: the full model defined by \eqref{eqn_full}.
\end{itemize}
Note that we use the same value of $\lambda$
in~\eqref{eqn_full.Dantzig} as in \eqref{eqn_ourformulation}, since in
both cases they lead to a constraint that the true signal $x^*$
satisfies $\|\tp^T(\tp x^* - \tilde{y})\|_\infty \leq \lambda\Delta/2$
with a certain probability; see~\eqref{eqn_full.Dantzig}
and~\eqref{eqn_feasible}. Readers familiar with the equivalence
between LASSO and Dantzig selector \citep{Bickel09} may notice that
\ltwodantziginf has similar theoretical error bounds to \lassoinf. Our
computational results show that the practical performance of these two
approaches is also similar.

The synthetic data is generated as follows. The measurement matrix
$\tp\in \mathbb{R}^{M\times N}$ is a Gaussian matrix, each entry being
independently generated from $\mathcal{N}(0, 1/R^2)$, for a given
parameter $R$. The $S$ nonzero elements of $x^*$ are in random
locations and their values are drawn from independently from
$\mathcal{N}(0,1)$.  We use
$\mbox{SNR}=-20\log_{10}(\|\hat{x}-x^*\|/\|x^*\|)$ as the error
metric, where $\hat{x}$ is the signal recovered from each of the
formulations under consideration. Given values of saturation parameter
$G$ and number of bits $B$, the interval $\Delta$ is defined
accordingly as $\Delta=2^{B-1}G$. All experiments are repeated $30$
times; we report the average performance.

We now describe how the bounds $\lambda$ for \eqref{eq_obj} and
\eqref{eqn_full.Dantzig} and $\epsilon$ for \eqref{eqn_full.L2}
\jlresolved{{\rca I saw that you changed the labels of Eq(16). Do
you mean (16b) and (16d)?} SJW: I guess, but since $\lambda$ also
  appears in \eqref{eq_obj}, don't we need to mention it here also? Do
  we use the same technique to choose this earlier $\lambda$?  {\rca
    Yes, we use the same way to choose both of them.}} were chosen for
these experiments. Essentially, $\epsilon$ and $\lambda$ should be
chosen so that the constraints~\eqref{eqn_full.L2}
and~\eqref{eqn_full.Dantzig} admit the true signal $x^*$ with a a
high (specified) probability. There is a tradeoff between tightness
of the error estimate and confidence.  Larger values of $\epsilon$
and $\lambda$ can give a more confident estimate, since the defined
feasible region includes $x^*$ with a higher probability, while
smaller values provide a tighter estimate. Although
Lemma~\ref{lem:feasible} suggests how to choose $\lambda$ and
\cite{JacquesHF11} show how to determine $\epsilon$, the analysis it
not tight, especially when $M$ and $N$ are not particularly large.
We use instead an approach based on simulation and on making the
assumption (not used elsewhere in the analysis) that the
non-saturated quantization errors $\xi_i=(\tp x^*-\tilde{y})_i$ are
i.i.d. uniform in $U_{[-\Delta/2,\Delta/2]}$. (As noted earlier,
this stronger assumption makes sense in some settings, and has been
used in previous analyses.)
We proceed by generating numerous independent samples of $Z\sim
U_{[-\Delta/2,\Delta/2]} $. Given a confidence level $1-\pi$ (for
$\pi>0$), we set $\epsilon$ to the value for which $\mathbb{P}(Z\geq
\epsilon\Delta) = \pi$ is satisfied empirically. A similar technique
is used to determine $\lambda$. When we seek certainty ($\pi=0$, or
confidence $P=100\%$), we set $\epsilon$ and $\lambda$ according to
the true solution $x^*$, that is, $\epsilon=\|\tp
x^*-\tilde{y}\|/\Delta$ and $\lambda=2\|\tp^T(\tp
x^*-\tilde{y})\|_\infty/\Delta$.

To summarize the parameters that are varied in our experiments: 
\begin{itemize}
\item $M$ and $N$ are dimensions of $\Phi$,
\item $S$ is sparsity of solution $x^*$,
\item $G$ is saturation level, 
\item $B$ is number of bits, 
\item $R$ is the inverse standard deviation of the elements of $\Phi$,
  and
\item $P=1-\pi$ denotes the confidence levels, expressed as a
  percentage.
\end{itemize}

%

In Figure~\ref{fig_N}, we fix the values of $M$, $S$, $G$, $R$, and
$P$, choose two values of $B$: 3 and 5.  Plots show the average SNRs
(over $30$ trials) of the solutions $\hat{x}$ recovered from the five
models against the dimension $N$.  In this and all subsequent figures,
the saturation ratio is defined to be $\bar{M}/M = (M-\tilde{M}) / M$,
the fraction of extreme measurements. Our \lassoinf formulation and
the full model \ltwodantziginf give the best recovery performance for
small $N$, while for larger $N$, \lassoinf is roughly tied with the
the \ltwo model. The \linf and \dantzig models have poorer
performance, a pattern that we continue to observe in subsequent
tests.

\begin{figure}[h]
  \centering
    \subfigure{\includegraphics[scale=0.31]{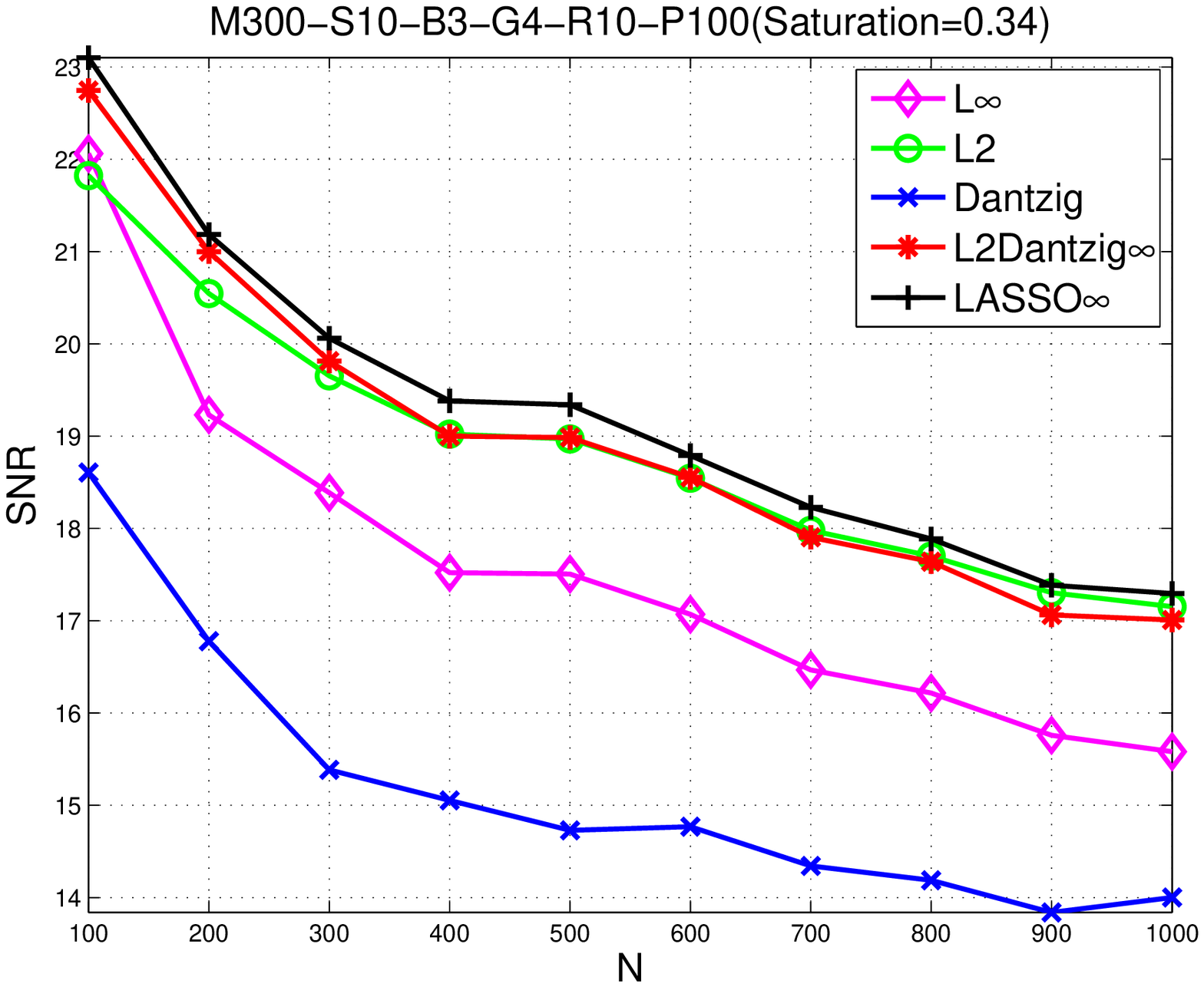}}
    \subfigure{\includegraphics[scale=0.31]{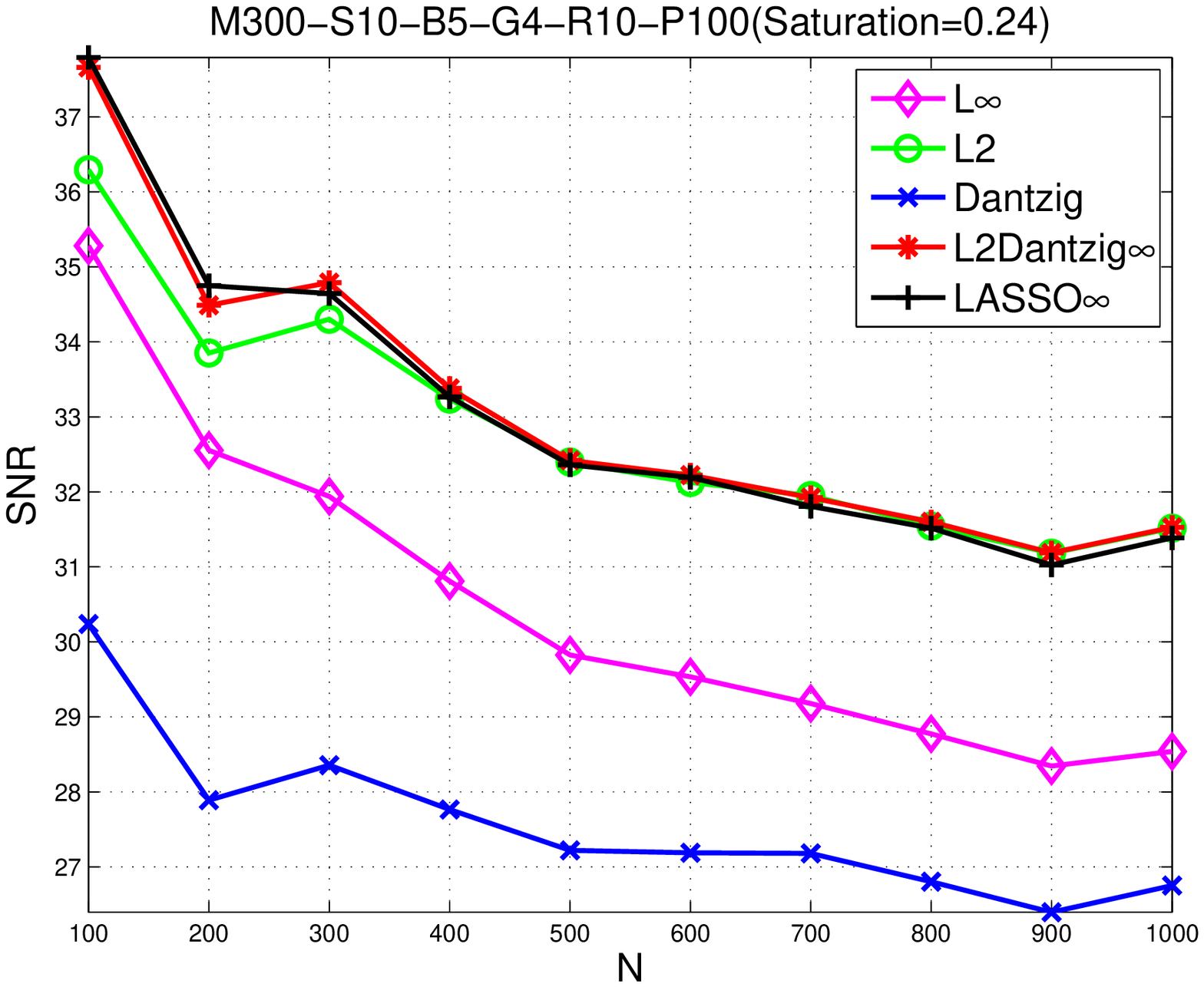}}
    \caption{Comparison among various models for fixed values
      $M=300$, $S=10$, $G=4$, $R=10$, and $P=100\%$, and two values
      of $B$ (3 and 5, respectively). The graphs show dimension
      $N$ (horizontal axis) against SNR (vertical axis) for values of
      $N$ between $100$ and $1000$, averaged over $30$ trials for each
      combination of parameters.}
    \label{fig_N}
\end{figure}

Figure~\ref{fig_S} fixes $N$, $M$, $B$, $G$, $R$, and $P$, and plots
SNR as a function of sparsity level $S$.  For all models, the
quality of reconstruction decreases rapidly with $S$. \lassoinf and
\ltwodantziginf achieve the best results overall, but are roughly
tied with the \ltwo model for all but the sparsest signals. The
\linf model is competitive for very sparse signals, while the
\dantzig model lags in performance.

\begin{figure}
  \centering
    \subfigure{\includegraphics[scale=0.31]{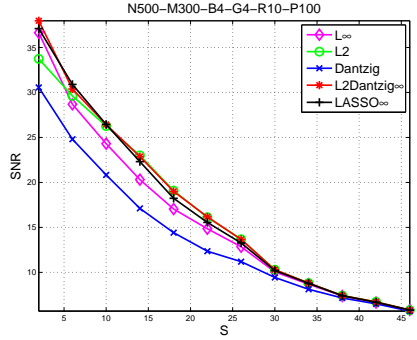}}
    \caption{Comparison among various models for $N=500$, $M=300$,
      $B=4$, $G=0.4$, $R=10$, and $P=100\%$. The graph shows sparsity
      level $S$ (horizontal axis) plotted against SNR (vertical axis),
      averaged over $30$ trials.}
    \label{fig_S}
\end{figure}

We now examine the effect of number of measurements $M$ on SNR.
Figure~\ref{fig_M1} fixes $N$, $S$, $G$, $R$, and $P$, and tries two
values of $B$: $3$ and $5$, respectively. Figure~\ref{fig_M2} fixes
$B=4$, and allows $N$ to increase with $M$ in the fixed ratio $5/4$.
These figures indicate that the \lassoinf and \ltwodantziginf models
are again roughly tied with the \ltwo model when the number of
measurements is limited. For larger $M$, our models have a slight
advantage over the \ltwo and \linf models, which is more evident
when the quantization intervals are smaller (that is, $B=4$).
Another point to note from Figure~\ref{fig_M2} is that \linf
outperforms \ltwo when both $M$ and $N$ are much larger than the
sparsity $S$.


\begin{figure}[h]
  \centering
    \subfigure{\includegraphics[scale=0.31]{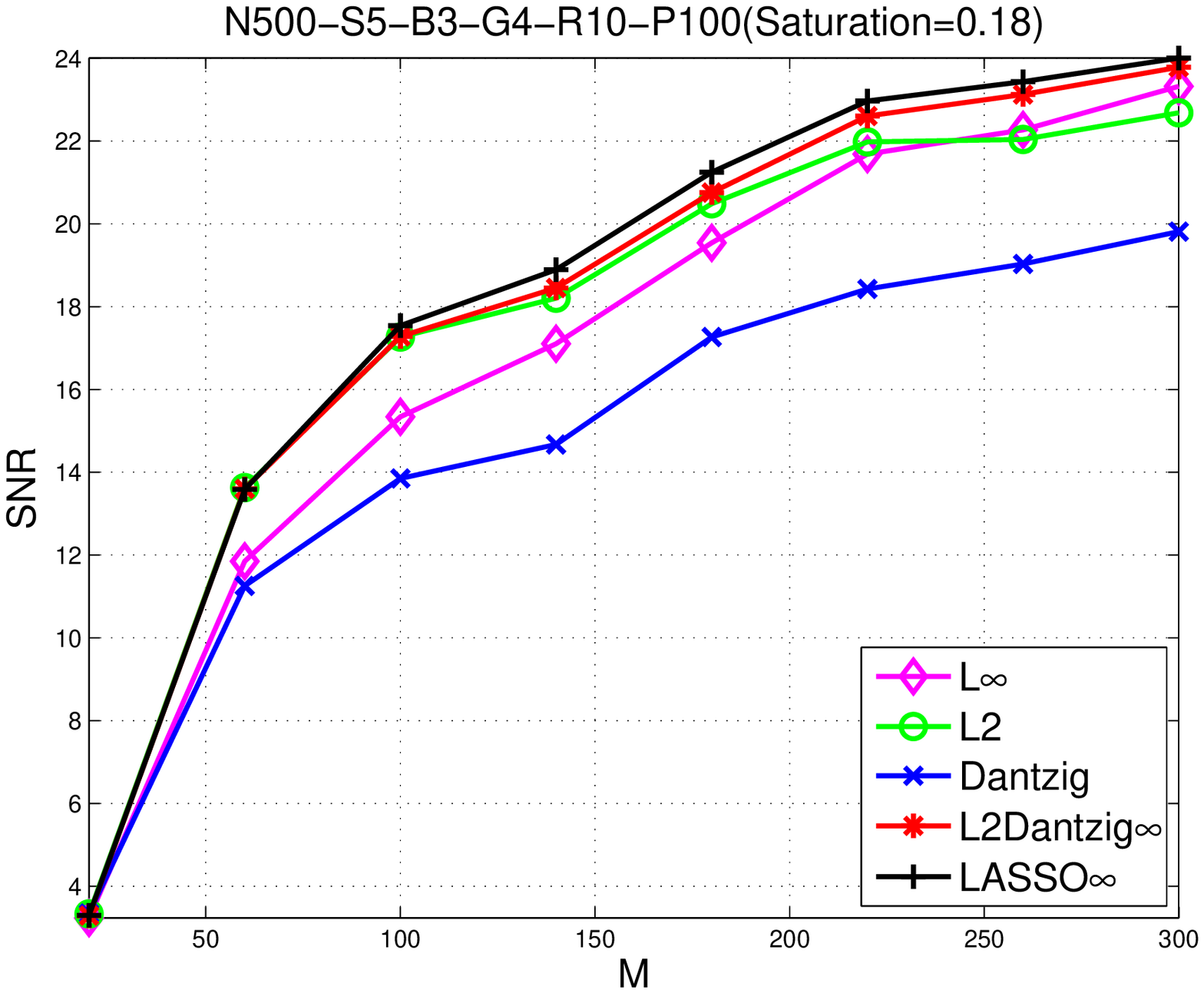}}
    \subfigure{\includegraphics[scale=0.31]{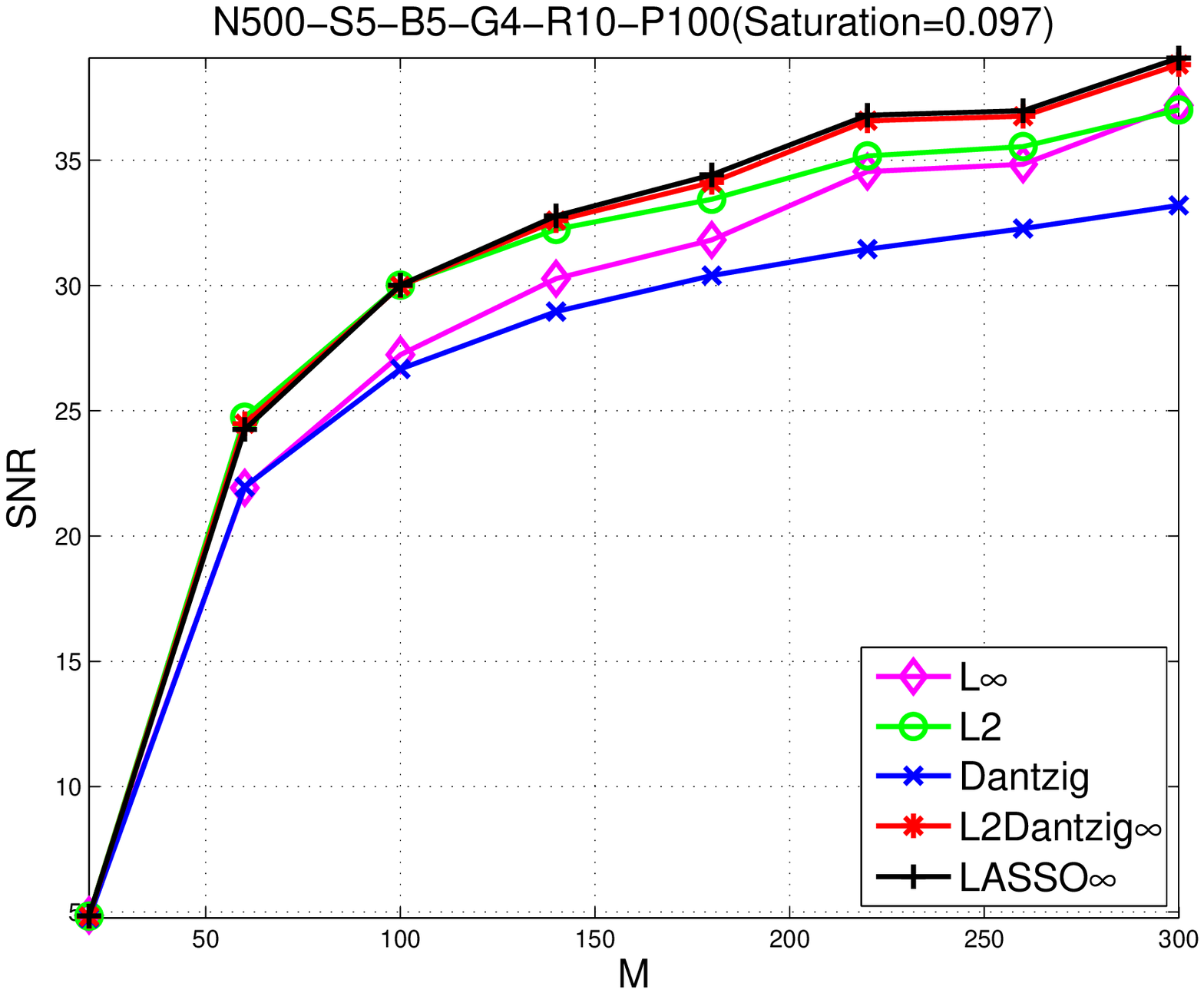}}
\caption{Comparison among various models for fixed values $N=500$,
$S=5$, $G=0.4$, $R=15$, and $P=100\%$, and two values of $B$ ($3$
and $5$). The graphs show the number of measurements $M$ (horizontal
axis) against SNR (vertical axis) for values of $M$ between $20$ and
$300$, averaged over $30$ trials for each combination of
parameters.}
    \label{fig_M1}
\end{figure}
\begin{figure}[h]
  \centering
     \subfigure{\includegraphics[scale=0.31]{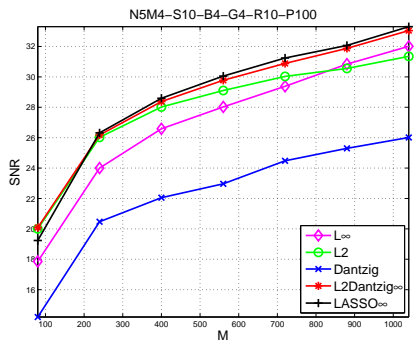}}
     \caption{Comparison among various models for fixed ratio
       $N/M=5/4$, and fixed values $S=10$, $B=4$, $G=0.4$, $R=15$, and
       $P=100\%$. The graph shows the number of measurements $M$
       (horizontal axis) against SNR (vertical axis) for values of $M$
       between $100$ and $1680$, averaged over $30$ trials for each
       combination of parameters.}
    \label{fig_M2}
\end{figure}

In Figure~\ref{fig_B} we examine the effect of the number of bits
$B$ on SNR, for fixed values of $N$, $M$, $S$, $G$, $R$, and $P$.
The fidelity of the solution from all models increases linearly with
$B$, with the \lassoinf, \ltwodantziginf, and \ltwo models being
slightly better than the alternatives.

\begin{figure}[h]
  \centering
     \subfigure{\includegraphics[scale=0.31]{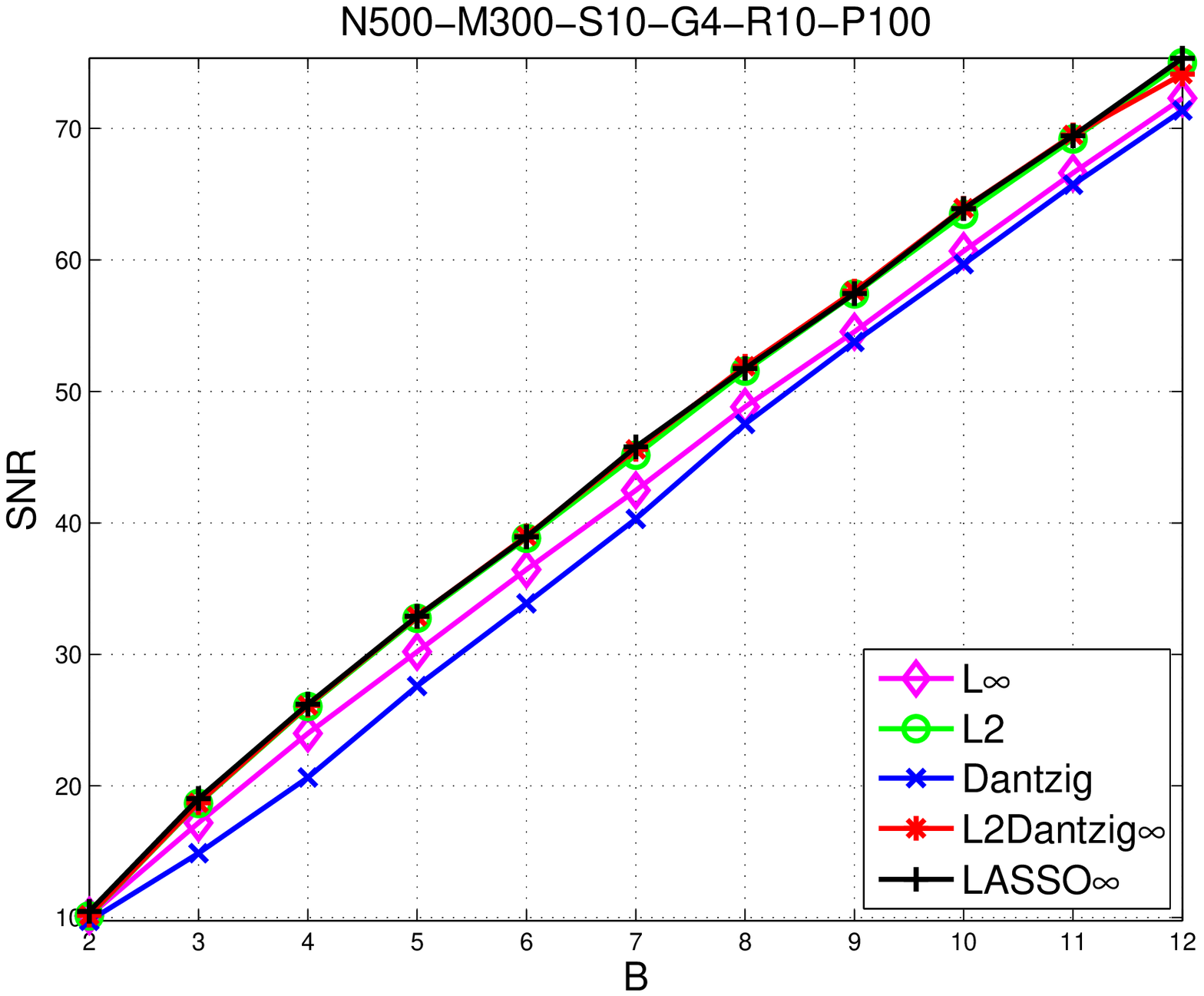}}
     \caption{Comparison among various models for fixed values $N=500$,
$M=300$, $S=10$, $G=0.4$, $R=10$, and $P=100\%$. This graph shows
the bit number $B$ (horizontal axis) against SNR (vertical axis),
averaged over $30$ trials.}
    \label{fig_B}
\end{figure}

Next we examine the effect on SNR of the confidence level, for fixed
values of $N$, $M$, $B$, $G$, and $R$. In Figure~\ref{fig_P1}, we set
$M=300$ and plot results for two values of $S$: 5 and 15. In
Figure~\ref{fig_P2}, we use the same values of $S$, but set $M=150$
instead.  Note first that the confidence level does not affect the
solution of the \linf model, since this is a deterministic model, so
the reconstruction errors are constant for this model. For the other
models, we generally see degradation as confidence is higher, since
the constraints \eqref{eqn_full.L2} and \eqref{eqn_full.Dantzig} are
looser, so the feasible point that minimizes the objective $\| \cdot
\|_1$ is further from the optimum $x^*$. Again, we see a clear
advantage for \lassoinf when the sparsity is low, $M$ is larger, and
the confidence level $P$ is high. For less sparse solutions, the
\ltwo, \ltwodantziginf, and \lassoinf models have similar or better
performance.  In addition, we find that \lassoinf is more robust to
the choice of confidence parameter than other methods (see also
Figure~\ref{fig_PG}), although this feature of the method is not
evident from our theoretical analysis.


\begin{figure*}[h]
  \centering
    \subfigure{\includegraphics[scale=0.31]{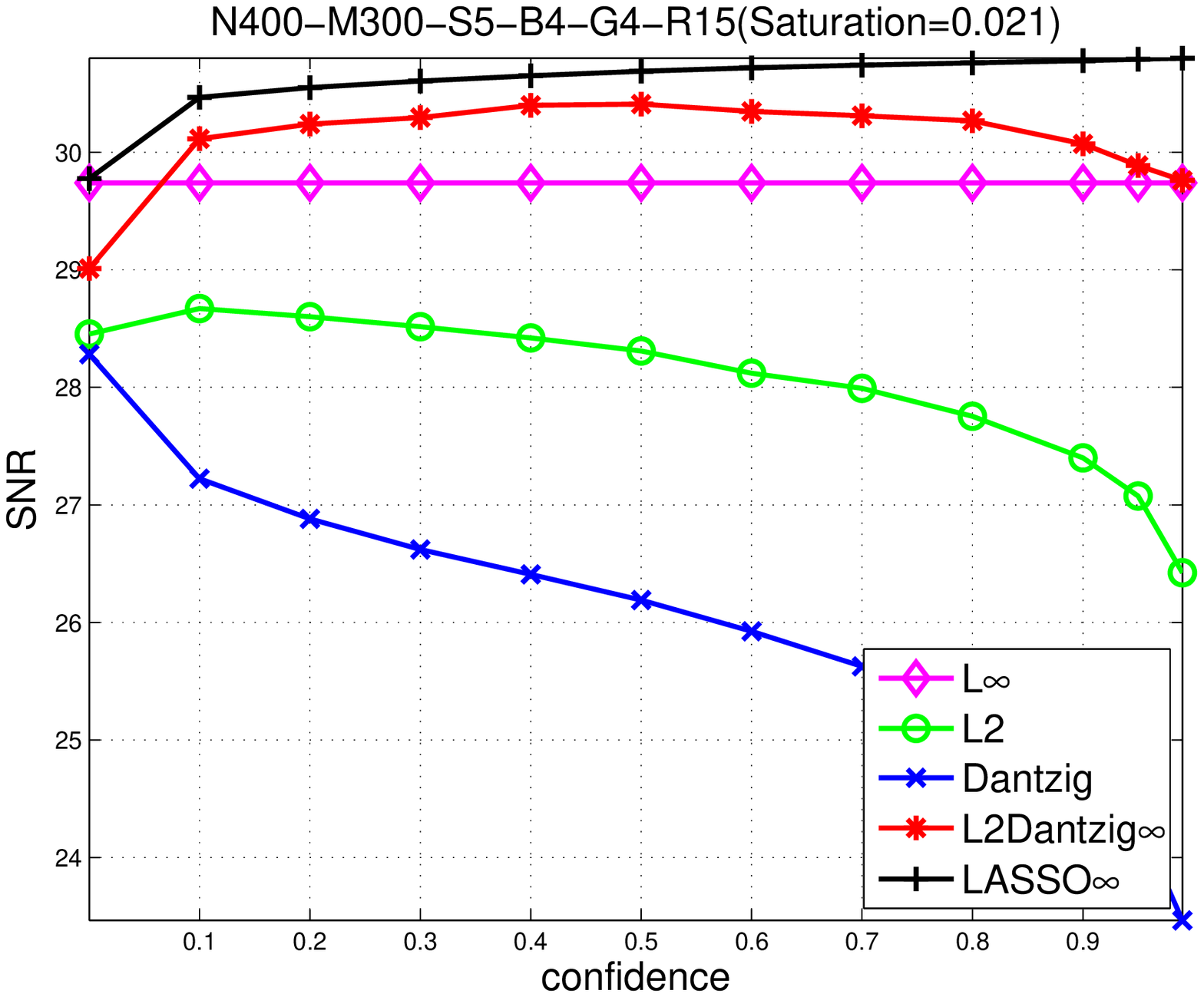}}
    \subfigure{\includegraphics[scale=0.31]{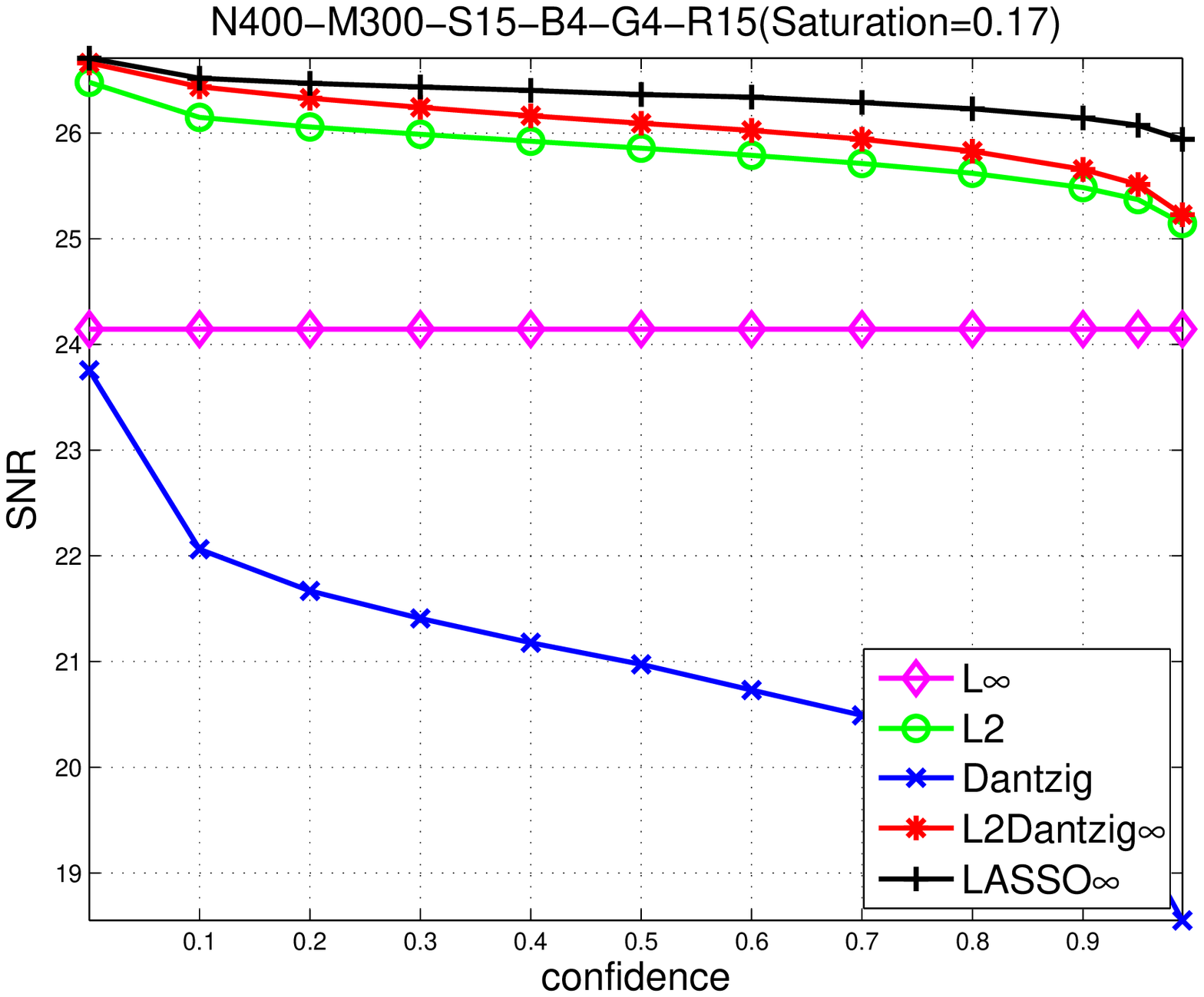}}
    \caption{Comparison among various models for fixed values $N=400$,
      $M=300$, $B=4$, $G=0.4$, and $R=15$, and sparsity levels $S=5$
      and $S=15$.  The graphs show saturation bound $G$ (horizontal
      axis) against SNR (vertical axis) for values of $P$ between
      $0.0001$ and $0.99$, averaged over $30$ trials for each
      combination of parameters.}
    \label{fig_P1}
\end{figure*}

\begin{figure*}[h]
  \centering
    \subfigure{\includegraphics[scale=0.31]{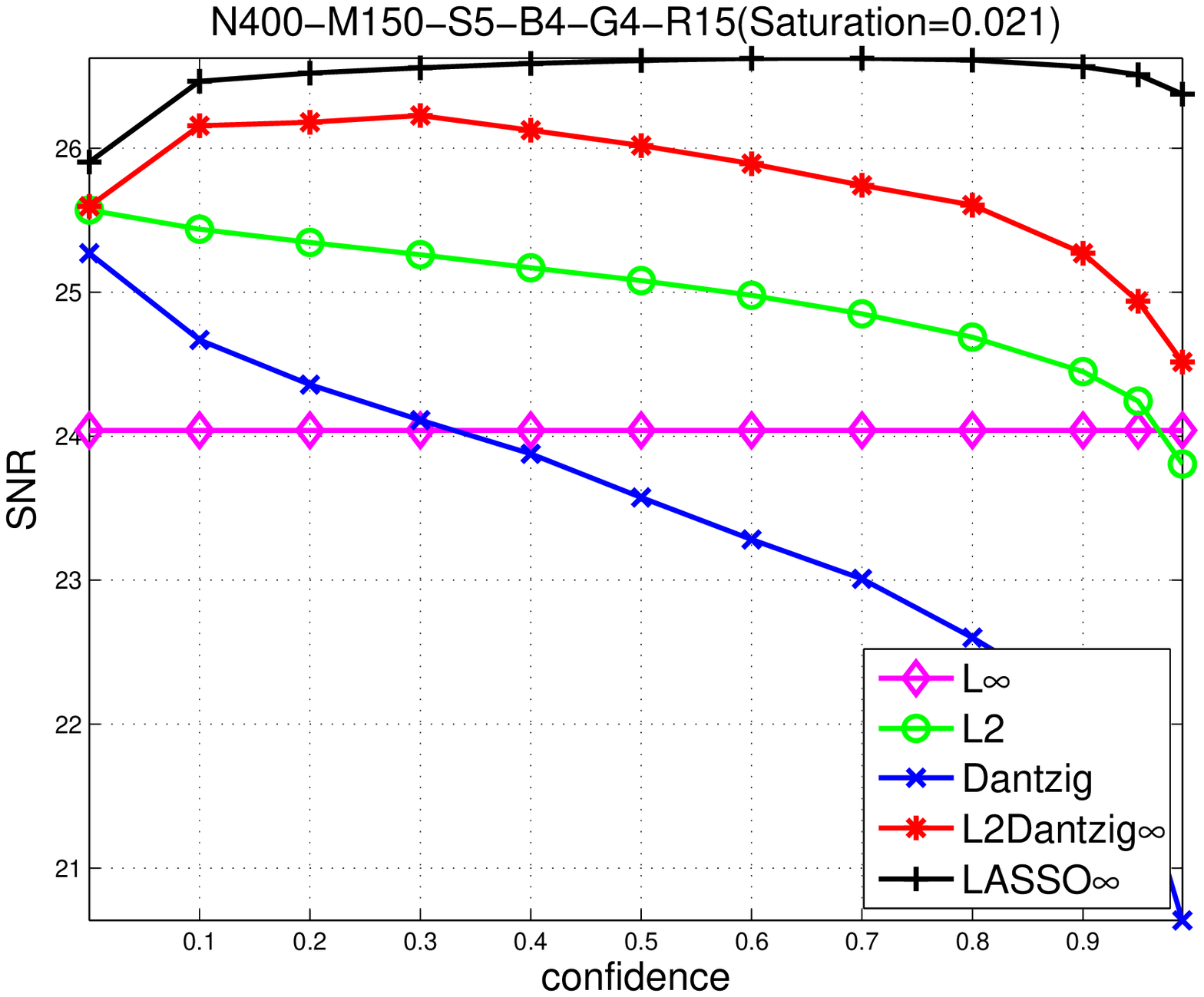}}
    \subfigure{\includegraphics[scale=0.31]{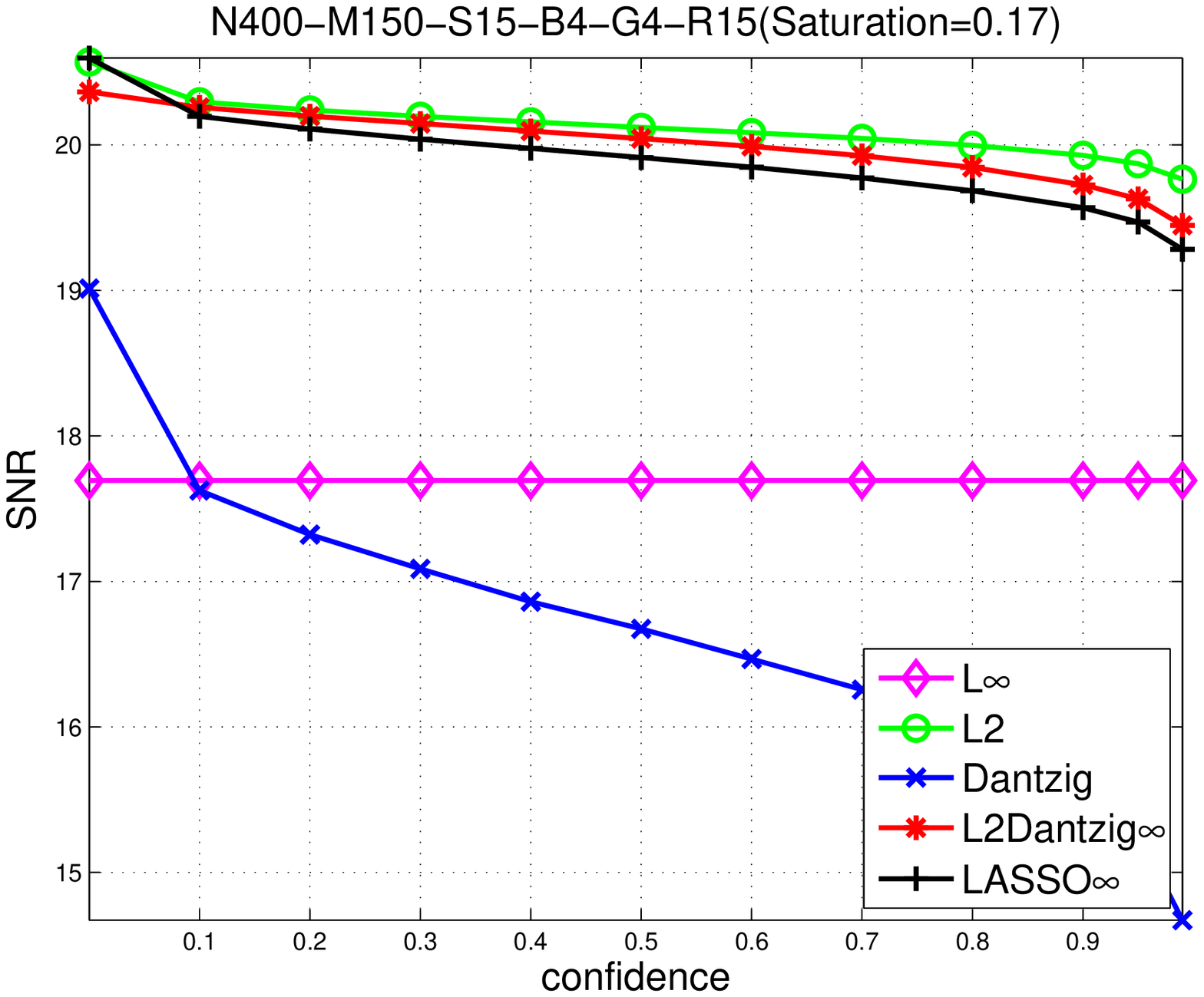}}
\caption{Comparison among various models for fixed values $N=400$,
  $M=150$, $B=4$, $G=0.4$, and $R=15$, and sparsity levels $S=5$ and
  $S= 15$.  The graphs show confidence $P$ (horizontal axis) against
  SNR (vertical axis) for values of $P$ between $0.0001$ and $0.99$,
  averaged over $30$ trials for each combination of parameters.}
    \label{fig_P2}
\end{figure*}

In Figure~\ref{fig_GS} we examine the effect of saturation bound $G$
on SNR. We fix $N$, $M$, $B$, $R$, and $P$, and try two values of $S$:
$5$ and $10$.
A tradeoff is evident --- the reconstruction performances are not
monotonic with $G$. As $G$ increases, the proportion of saturated
measurements drops sharply, but the quantization interval also
increases, degrading the quality of the measured observations. We
again note a slight advantage for the \lassoinf and \ltwodantziginf
models, with very similar performance by \ltwo when the oversampling
is lower.


\begin{figure*}[h]
  \centering
    \subfigure{\includegraphics[scale=0.31]{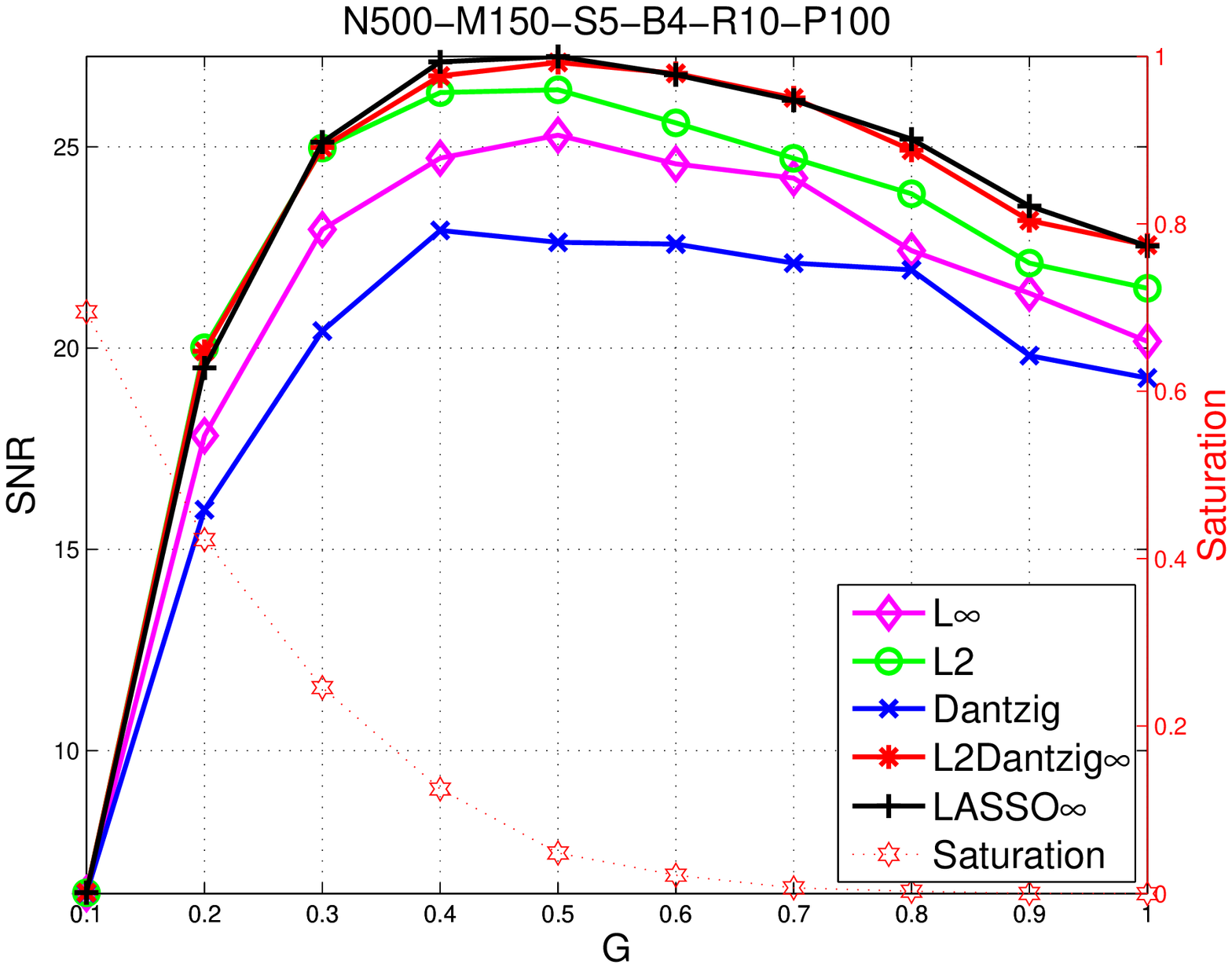}}
    \subfigure{\includegraphics[scale=0.31]{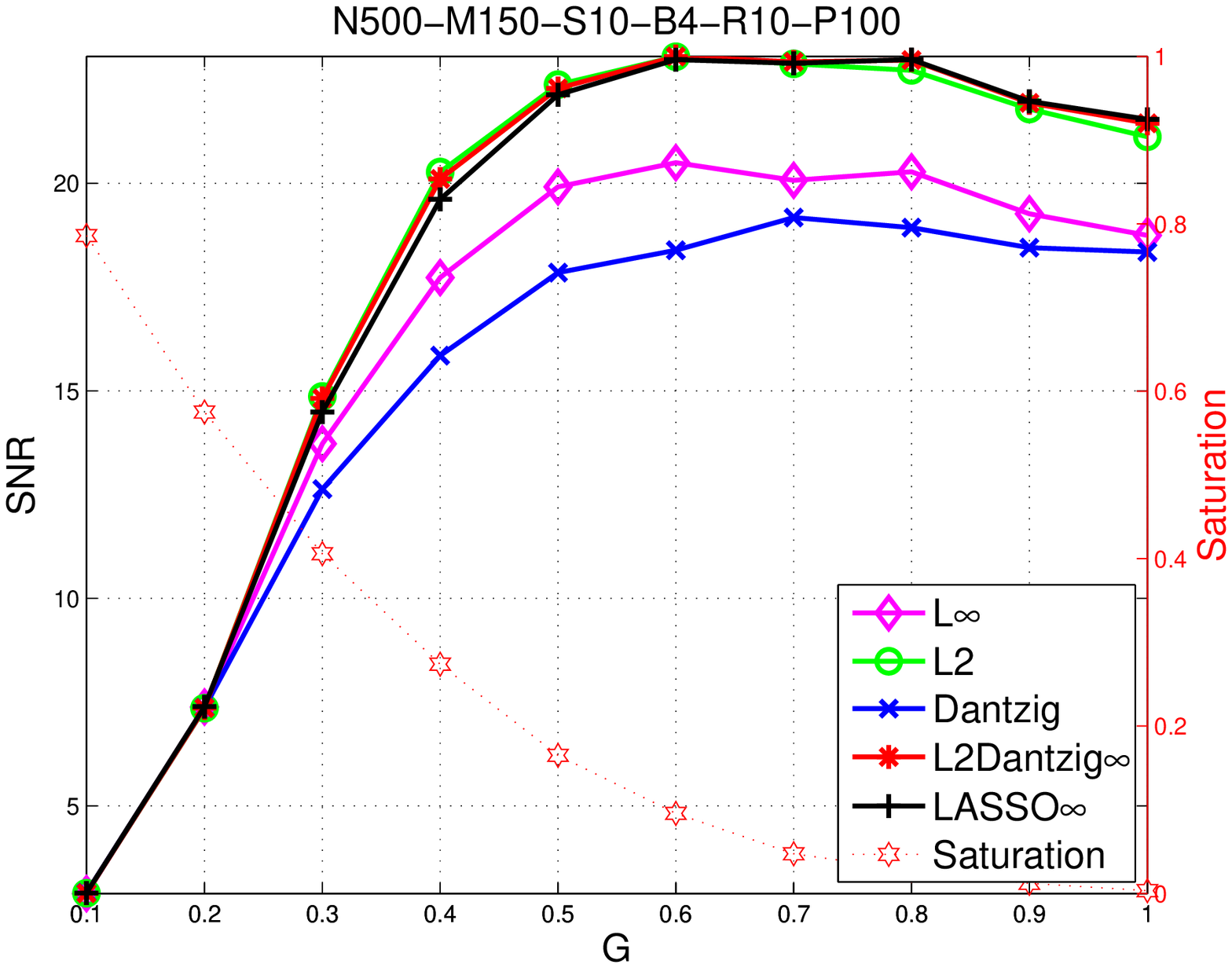}}
    \caption{Comparison among various models for fixed values of $N=500$, $M=150$,
      $B=4$, $R=15$, $P=100\%$, and two values of $S$: $5$ and $10$. 
      The graphs show confidence $P$ (horizontal axis) against
      SNR (left vertical axis) and saturation ratio (right vertical axis), averaged
      over $30$ trials for each combination of parameters.}
\label{fig_GS}
\end{figure*}

In Figure~\ref{fig_PG}, we fix $N$, $M$, $S$, $B$, $R$, and tune the
value of $G$ to achieve  specified saturation ratios of $2\%$ and $10\%$.
We plot SNR against the confidence level $P$, varied from $0\%$ to
$100\%$.  Again, we see generally good performance from the \lassoinf
and \ltwodantziginf models, with \ltwo being competitive for less
sparse solutions.

\begin{figure*}[h]
  \centering
    \subfigure{\includegraphics[scale=0.31]{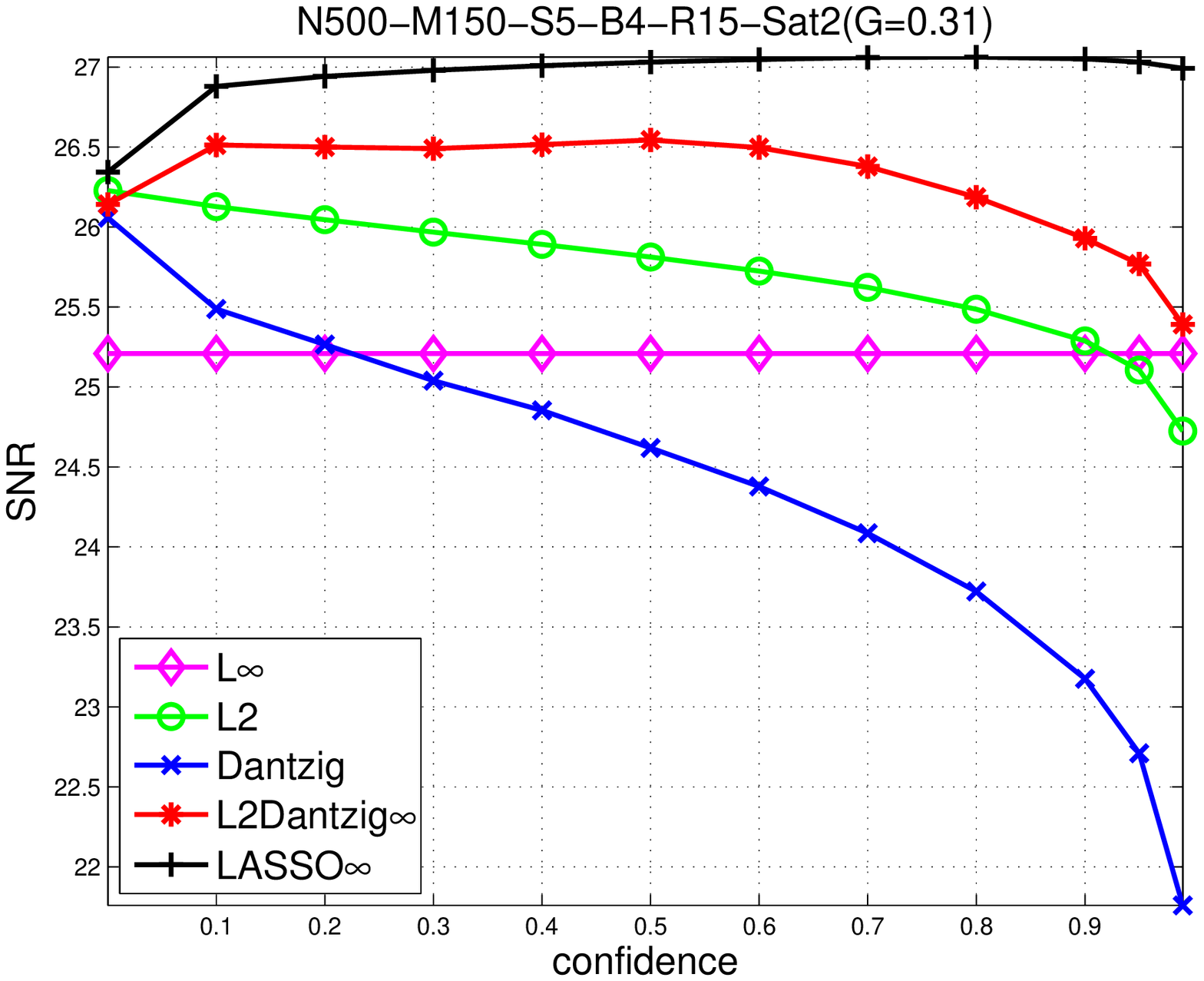}}
    \subfigure{\includegraphics[scale=0.31]{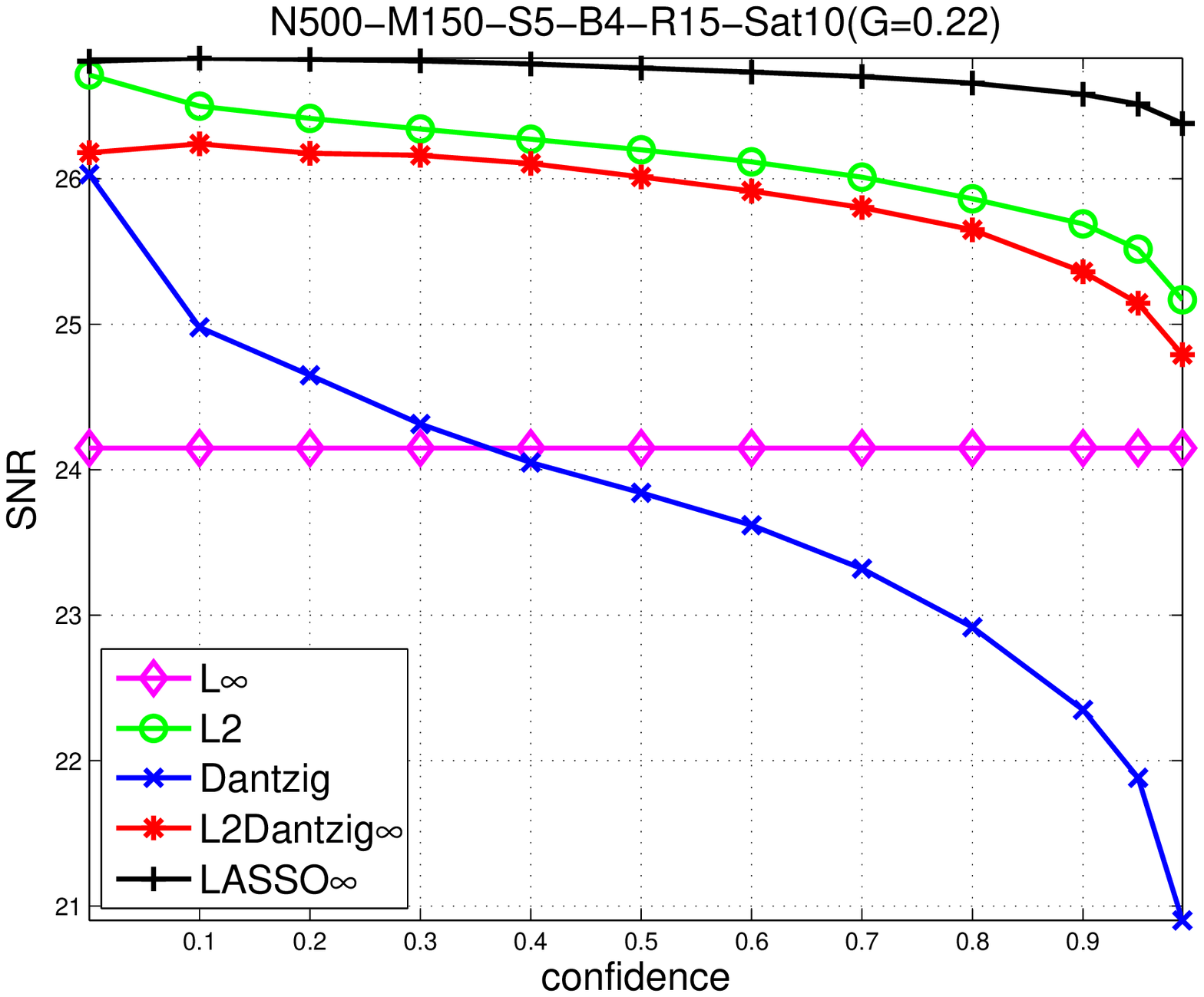}}
    \caption{Comparison among various models for fixed values of
      $N=500$, $M=150$, $S=5$, $B=4$, $R=15$, and two values of
      saturation ratio: $2\%$ and $10\%$, which are achieved by tuning
      the value of $G$. The graphs show confidence $P$ (horizontal
      axis) against SNR (vertical axis), averaged over $30$ trials for
      each combination of parameters.}
    \label{fig_PG}
\end{figure*}

Summarizing, we note the following points.
\begin{itemize}
\item[(a)] Our proposed \lassoinf formulation gives either best or
  equal-best reconstruction performance in most regimes, with a more
  marked advantage when the signal is highly sparse and the number of
  samples is higher.
\item[(b)] The \ltwo model has similar performance to the full model,
  and is even slightly better than our model for less sparse signals
  with fewer measurements, since it is not sensitive to the
  measurement number as the upper bound suggested
  by~\cite{Laska11}. Although the inequality
  in~\eqref{eqn_boundsimple} also indicates the estimate error by our
  model is bounded by a constant due to the $\ell_\infty$ constraint,
  the error bound determined by the $\ell_\infty$ constraint is not as
  tight as the $\ell_2$ constraint in general. This fact is evident
  when we compare the the \linf model with the \ltwo
  model. \sjwresolved{There is no $B_2$ in Theorem 1 or in fact
    anywhere in the paper, so we had better rewrite this item. Can you
    do this? . {\rca I rewrote this part}}
\item[(c)] The \linf model performs well (and is competitive with the
  others) when the number of unsaturated measurements is relatively
  large.
\item[(d)] The \ltwodantziginf model is competitive with \lassoinf if
  $\epsilon$ and $\lambda$ can be determined from the true signal
  $x^*$. Otherwise, \lassoinf is more robust to choices of these
  parameters that do not require knowledge of the true signals,
  especially if a high confidence level is desired.\sjwresolved
{I am
    not clear about this argument. Doesn't \lassoinf also need us to
    supply $\lambda$, and didn't you use the ``nice'' simulated value
    for it here? So how do we know it is more robust to the choice of
    $\lambda$? {\rca \lassoinf needs to set $\lambda$ like the full
      model. ``$\epsilon$ and $\lambda$ can be exactly determined''
      means that they are determined by the true signal $x^*$. I want
      to say if one chooses $\epsilon$ and $\lambda$ using simulation,
      the \lassoinf model is much robust than the full model
      especially when a high confidence level is desired.}}
\end{itemize}


\section{Conclusion}
\label{conclusion} We have analyzed a formulation of the
reconstruction problem from compressed sensing in which the
measurements are quantized to a finite number of possible values.
Our formulation uses an objective of $\ell_2$-$\ell_1$ type, along
with explicit constraints that restrict the individual quantization
errors to known intervals. We obtain bounds on the estimation error,
and estimate these bounds for the case in which the sensing matrix
is Gaussian. Finally, we prove the practical utility of our
formulation by comparing with an approach that has been proposed
previously, along with some variations on this approach that attempt
to distil the relative importance of different constraints in the
formulation.

\section*{Acknowledgments}
The authors acknowledge support of National Science Foundation Grant
DMS-0914524 and a Wisconsin Alumni Research Foundation 2011-12 Fall
Competition Award. The authors are also grateful to the editor and
three referees whose constructive comments on the first version led to
improvements in the manuscript.

\appendix
\section{}

This section contains the proof to a more general form of
Theorem~\ref{thm_main}, developed via a number of technical lemmas. At
the end, we state and prove a result (Theorem~\ref{thm_rhobound})
concerning high-probability estimates of the bounds under additional
assumptions on the sensing matrix $\tp$.

Theorem~\ref{thm_main} is a corollary of the following more general
result.
\begin{theorem} \label{thm_main_extend}
Assume that the true signal $x^*$ satisfies
\begin{equation}\label{eqn_feasible_extend}
\|\tp^T(\tp x^* - \tilde{y})\|_\infty \leq \lambda\Delta /2,
\end{equation}
for some value of $\lambda$. Let $s$ and $l$ be positive integers in
the range $1,2,\dotsc,N$, and define
\begin{subequations}
\begin{align}
\label{eq:def.A0.extend}
A_0(\Psi):=&{\rho}^-(s+l, \Psi)-3{\sqrt{s/l}} \left[ {\rho}^+(s+2l, \Psi)-{\rho}^-(s+2l, \Psi) \right]\\
\label{eq:def.A1.extend}
A_1(\Psi):=&4[{\rho}^+(s+2l, \Psi)-{\rho}^-(s+2l,\Psi)],\\
C_1(\Psi):=&4+{\sqrt{(1+9s/l)}A_1(\Psi)}/{A_0(\Psi)},\label{eq:def.C1.extend}\\
C_2(\Psi):=&\sqrt{(1+9s/l) / A_0(\Psi)} \label{eq:def.C2.extend}.
\end{align}
\end{subequations}
\sjwresolved{I don't think $B_S$ appears in the analysis any more.
Can
  we delete it? {\rca Yes, you are right. It should not appear in the
    theorem. Actually, I found a mistake about the bound caused by the
    saturation constraint in this morning. I deleted the inequality
    about it, but have no time to check if there is something
    inconsistent.} Have you checked and is it OK now? {\rca Yes, I have deleted all $B_S$ in this draft.}}
We have that for any $T_0\subset \{1,2,...,N\}$ with $s=|T_0|$, if
$A_0(\tp)>0$, then
\begin{subequations}
\begin{align}
\|h\| \leq& \frac{6C_2(\tp)^2\sqrt{s}\lambda\Delta}{\sqrt{1+9s/l}} + \frac{C_1(\tp)}{\sqrt{l}}  \|x^*_{T_0^c}\|_1 +
2.5{C_2(\tp)}\sqrt{\lambda \Delta \|x_{T_0^c}^*\|_1},
\label{eqn_thm:LASSO.1_extend} \\
\|h\|\leq & C_2(\tp)\sqrt{\tilde{M}}\Delta + \frac{C_1(\tp)}{\sqrt{l}} \|x^*_{T_0^c}\|_1.
\label{eqn_thm:LASSO.2_extend}
\end{align}
\end{subequations}

Suppose that Assumption~\ref{ass:1} holds, and let $\pi \in (0,1)$
be given. If we define $\lambda = \sqrt{2\log {2N/\pi}}\fmax$
in~\eqref{eqn_ourformulation}, then with probability at least
$P=1-\pi$, the inequalities \eqref{eqn_thm:LASSO.1_extend} and
\eqref{eqn_thm:LASSO.2_extend} hold.
\end{theorem}

Theorem~\ref{thm_main} can be proven by setting $s=l$ in
Theorem~\ref{thm_main_extend} and defining $\bar{C}_1(\tp)$ to be
$C_1(\Psi)$ for $l=s$ and $\Psi=\tp$, and similarly for
$\bar{C}_1(\tp)$, $\bar{A}_0(\tp)$, and $\bar{A}_1(\tp)$.

The proof of Theorem~\ref{thm_main_extend} essentially follows the
standard analysis procedure in compressive sensing. Some similar
lemmas and proofs can be found in
\citet{Bickel09,Candes07a,Candes08,Zhang09a,Liu10,Liu12JMLR}. For completeness,
we include all proofs in the following discussion.

  Given the error vector $h=\hat{x}-x^*$ and the set $T_0$ (with $s$
  entries), divide the complementary index set
  $T_0^c:=\{1,2,...,N\}\backslash T_0$ into a group of subsets $T_j$'s
  ($j=1,2,\dotsc,J$), without intersection, such that $T_1$ indicates
  the index set of the largest $l$ entries of $h_{T_0^c}$, $T_2$
  contains the next-largest $l$ entries of $h_{T_0^c}$, and so
  forth.\footnote{The last subset may contain fewer than $l$
    elements.}

\begin{lemma}
We have
\begin{align}
\|\tp h\|_\infty &\leq \Delta. \label{eq:lem1.1}
\end{align}
\label{lem_diff}
\end{lemma}
\begin{proof}
From \eqref{eq_Linfty_cst}, and invoking feasibility of $\hat{x}$
and $x^*$, we obtain
\[
\|\tp h\|_\infty =\|\tp (\hat{x}-x^*)\|_\infty
 \leq \|\tp \hat{x}-\tilde{y}\|_\infty + \|\tp x^*-\tilde{y}\|_\infty \leq \Delta.
\]
\end{proof}

\begin{lemma} \label{lem:feasible}
Suppose that Assumption~\ref{ass:1} holds. Given $\pi \in (0,1)$,
the choice $\lambda = \sqrt{2\log {(2N/\pi)}}\fmax$ ensures that the
true signal $x^*$ satisfies \eqref{eqn_feasible_extend}, that is
\[
  \|\tp^T(\tp x^* - \tilde{y})\|_\infty \leq \lambda\Delta /2
  \label{eqn_lemma:assumption}
\]
with probability at least $1-\pi$.
\end{lemma}
\begin{proof}
Define the random variable $Z_j=\tp_{j}^T(\tp
x^*-\tilde{y})=\tp_{j}^T\xi$, where
$\xi=[\xi_1,...,\xi_{\tilde{M}}]$ is defined in an obvious way.
(Note that $\|Z\|_\infty= \|\tp^T(\tp x^*-\tilde{y})\|_\infty$.)
Since $\mathbb{E}(Z_j)=0$ (from Assumption~\ref{ass:1}) and all
$\tp_{ij}\xi_i$'s are in the range $[-\tp_{ij}\Delta/2,
  \tp_{ij}\Delta/2]$, we use the Hoeffding inequality to obtain
\begin{align*}
\mathbb{P}(Z_j > \lambda\Delta/2) =&\mathbb{P}(Z_j - \mathbb{E}(Z_j)
>
\lambda\Delta/2)\\
=&\mathbb{P}\left(\sum_{i=1}^{\tilde{M}}\tp_{ij}\xi_i - \mathbb{E}(Z_j) >
\lambda\Delta/2 \right)\\
\leq& \exp{-2(\lambda\Delta/2)^2\over \sum_{i=1}^{\tilde{M}}(\tp_{ij}\Delta)^2}\\
=& \exp{-\lambda^2\over 2\sum_{i}\tp_{ij}^2}\\
\leq& \exp{-\lambda^2\over 2\fmax^2},
\end{align*}
\sjwresolved{Are you absolutely sure that the Hoeffding inequality
has
  been invoked correctly here? The variance bound on $|Z_j|$ is
  correct, the factor $1/2$ is correct, etc? I am offline so cannot
  check. {\rca Yes, the factor $1/2$ is correct. I have double checked
    the Hoeffding inequality. Actually, the factor should $2$, but
    considering $\lambda/2$ the factor becomes $1/2$.}}  which implies
(using the union bound) that
\begin{align*}
\mathbb{P}(|Z_j| > \lambda\Delta/2) \leq 2\exp{-\lambda^2\over 2\fmax^2}
& \Rightarrow \mathbb{P}\left(\|Z\|_\infty = \max_j|Z_j| > \lambda\Delta/2 \right) \leq 2N\exp{-\lambda^2\over 2\fmax^2}\\
& \Rightarrow \mathbb{P}\left(\|Z\|_\infty > \sqrt{{1\over
2}\log{2N\over \pi}}\fmax\Delta\right) \leq \pi,
\end{align*}
\sjwresolved{Can we shorten the derivation above by just invoking the
  ``union bound''? Again I cannot check as I am offline. {\rca Yes, it
    can be shortened.}}  where the last line follows by setting
$\lambda$ to the prescribed value. This completes the proof.
\end{proof}
Similar claims with Gaussian (or sub-Guassian) noise assumption to Lemma~\eqref{lem:feasible} can be found
in \citet{Zhang09a,Liu12JMLR}.

\begin{lemma}
We have
\[
\|h_{T_{01}^c}\| \leq
\sum_{j=2}^J\|h_{T_j}\|\leq\|h_{T_{0}^c}\|_1/\sqrt{l},
\]
where $T_{01}=T_0\cup T_1$. \label{lem_T01}
\end{lemma}
\begin{proof}
First, we have for any $j\geq 1$ that
\[
\|h_{T_{j+1}}\|^2 \leq l\|h_{T_{j+1}}\|_\infty^2 \leq
l(\|h_{T_j}\|_1/l)^2 = \|h_{T_j}\|_1^2/l,
\]
because the largest value in $|h_{T_{j+1}}|$ cannot exceed the
average value of the components of $|h_{T_j}|$.  It follows that
\[
\|h_{T_{01}^c}\| \leq \sum_{j=2}^J\|h_{T_j}\| \leq \sum_{j=1}^{J-1}
\| h_{T_j} \|_1 / \sqrt{l} \leq\|h_{T_{0}^c}\|_1/\sqrt{l}.
\]
\end{proof}
Similar claims or inequalities to Lemma~\ref{lem_T01} can be found
in \citet{Zhang09a,Candes07a, Liu10}.

\begin{lemma}
Assume that~\eqref{eqn_feasible_extend} holds. We have
\begin{subequations}
\begin{align}
\label{eq:lem7.1}
\|h_{T_0^c}\|_1 & \leq 3\|h_{T_0}\|_1 + 4\|x^*_{T_0^c}\|_1, \\
\label{eq:lem7.2} \|h\| & \leq \sqrt{1+9s/l} \|h_{T_{01}}\| +
4\|x^*_{T_{0}^c}\|_1/\sqrt{l}.
\end{align}
\end{subequations}
\label{lem_h}
\end{lemma}

\begin{proof}
Since $\hat{x}$ is the solution of~\eqref{eqn_ourformulation}, we
have
\begin{align*}
0&\geq {1\over 2}\|\tp \hat{x} - \tilde{y}\|^2 - {1\over 2}\|\tp x^*
-\tilde{y}\|^2 + \lambda\Delta(\|\hat{x}\|_1 - \|x^*\|_1)\\
&\geq h^T\tp^T(\tp x^* -\tilde{y}) +
\lambda\Delta(\|\hat{x}\|_1 - \|x^*\|_1) \quad\quad (\text{by convexity of $(1/2) \|\tp x - \tilde{y}\|^2$}) \\
&= h^T\tp^T(\tp x^* -\tilde{y}) +
\lambda\Delta(\|\hat{x}_{T_0}\|_1-\|x^*_{T_0}\|_1 +
\|\hat{x}_{T_0^c}\|_1 - \|x^*_{T_0^c}\|_1)\\
&\geq -\|h\|_1\|\tp^T(\tp x^* -\tilde{y})\|_\infty +
\lambda\Delta(\|\hat{x}_{T_0}\|_1-\|x^*_{T_0}\|_1 +
\|\hat{x}_{T_0^c}\|_1 - \|x^*_{T_0^c}\|_1)\\
&\geq -\|h\|_1\lambda\Delta/2 +
\lambda\Delta(\|\hat{x}_{T_0}\|_1-\|x^*_{T_0}\|_1 +
\|\hat{x}_{T_0^c}\|_1+\|x^*_{T_0^c}\|_1 - 2\|x^*_{T_0^c}\|_1)~~~~(\text{from~\eqref{eqn_feasible_extend}})\\
&\geq -(\|h_{T_0}\|_1+\|h_{T_0^c}\|_1)\lambda\Delta/2 +
\lambda\Delta(-\|h_{T_0}\|_1 +
\|h_{T_0^c}\|_1 - 2\|x^*_{T_0^c}\|_1)\\
&= {1\over 2}\lambda\Delta\|h_{T_0^c}\|_1 - {3\over
2}\lambda\Delta\|h_{T_0}\|_1 - 2\lambda\Delta\|x^*_{T_0^c}\|_1.
\end{align*}
It follows that $3\|h_{T_0}\|_1 + 4\|x^*_{T_0^c}\|_1 \geq
\|h_{T_0^c}\|_1$, proving \eqref{eq:lem7.1}.

The second inequality \eqref{eq:lem7.2} is from
\begin{align*}
\|h\|^2 &= \|h_{T_{01}}\|^2 + \|h_{T_{01}^c}\|^2\\
&\leq \|h_{T_{01}}\|^2 + \|h_{T_{0}^c}\|^2_1/l~~~~(\text{from Lemma~\ref{lem_T01}})\\
&\leq \|h_{T_{01}}\|^2 + (3\|h_{T_0}\|_1 + 4\|x^*_{T_0^c}\|_1)^2/l~~~~(\text{from \eqref{eq:lem7.1}})\\
&\leq \|h_{T_{01}}\|^2 + (3\sqrt{s}\|h_{T_{01}}\| + 4\|x^*_{T_0^c}\|_1)^2/l\\
&= (1+9s/l)\|h_{T_{01}}\|^2 +
24\sqrt{s}/l\|h_{T_{01}}\|\|x^*_{T_0^c}\|_1
+ {16\|x^*_{T_0^c}\|_1^2/l} \\
&\leq \left[\sqrt{1+9s/l}\|h_{T_{01}}\| +
4\|x^*_{T_0^c}\|_1/\sqrt{l}\right]^2.
\end{align*}
\end{proof}

\begin{lemma}
For any matrix $\Psi$ with
$N$ columns, and $s,l\leq N$, we have
\[
\|\Psi h\|^2 \geq A_0(\Psi)\|h_{T_{01}}\|^2 -
A_1(\Psi)\|h_{T_{01}}\|\|x^*_{T_{0}^c}\|_1/\sqrt{l},
\]
where $A_0(\Psi)$ and $A(\Psi)$ are defined in \eqref{eq:def.A0} and
\eqref{eq:def.A1} respectively. \label{lem_left.LASSO}
\end{lemma}
\begin{proof}
For any $j\geq 2$, we have
\begin{align}
\nonumber
& \;\; {|h^T_{T_{01}}\Psi_{T_{01}}^T\Psi_{T_j} h_{T_j}| \over  \|h_{T_{01}}\|\|h_{T_j}\|} \\
\nonumber
&= \frac14 \left| \left\|\Psi_{T_{01}} h_{T_{01}}/\|h_{T_{01}}\|+\Psi_{T_j} h_{T_j}/\|h_{T_j}\|\|^2-\|\Psi_{T_{01}} h_{T_{01}}/\|h_{T_{01}}\|-\Psi_{T_j} h_{T_j}/\|h_{T_j}\| \right\|^2 \right|\\
&= \frac14 \left| \left\| \left[ \Psi_{T_{01}} \, : \, \Psi_{T_j}
\right] \left[ \begin{matrix} h_{T_{01}} / \|h_{T_{01}} \| \\
h_{T_j} / \|h_{T_j} \| \end{matrix} \right] \right\|^2-
\nonumber\left\| \left[ \Psi_{T_{01}} \, : \, \Psi_{T_j} \right]
\left[ \begin{matrix} h_{T_{01}} / \|h_{T_{01}} \| \\
-h_{T_j} / \|h_{T_j} \| \end{matrix} \right] \right\|^2 \right| \\
\nonumber
&\leq \frac14 \left( {2}{\rho}^+({s+2l})-{2}{\rho}^-(s+2l) \right)\\
\label{eqn_lem4} &=\frac12 \left( {\rho}^+(s+2l)-{\rho}^-(s+2l) \right).
\end{align}
The inequality above follows from the definitions \eqref{eq:defrho-}
and \eqref{eq:defrho+}, and the fact that fact that $h_{T_{01}}/\|
h_{T_{01}} \|$ and $h_{T_j} / \| h_{T_j} \|$ are $\ell_2$-unit
vectors, so that
\[
\left\| \left[ \begin{matrix} h_{T_{01}} / \|h_{T_{01}} \| \\
h_{T_j} / \|h_{T_j} \| \end{matrix} \right] \right\|^2 =
\left\| \left[ \begin{matrix} h_{T_{01}} / \|h_{T_{01}} \| \\
-h_{T_j} / \|h_{T_j} \| \end{matrix} \right] \right\|^2 = 2.
\]
Considering the left side of the claimed inequality, we have
\begin{align*}
  & \;\; \|\Psi h\|^2 \\
&= \|\Psi_{T_{01}}h_{T_{01}} \|^2 +
  2h^T_{T_{01}}\Psi^T_{T_{01}}\Psi_{T_{01}^c}h_{T_{01}^c} +
  \|\Psi_{T_{01}^c}h_{T_{01}^c}\|^2\\
  &\geq \|\Psi_{T_{01}}h_{T_{01}} \|^2 - 2\sum_{j\geq
  2}|h^T_{T_{01}}\Psi^T_{T_{01}}\Psi_{T_{j}}h_{T_{j}}|\\
&\geq {\rho}^-(s+l)\|h_{T_{01}}\|^2 - ({\rho}^+{(s+2l)}-{\rho}^-(s+2l))\|h_{T_{01}}\|\sum_{j\geq2}\|h_{T_j}\|~~~~(\text{from \eqref{eqn_lem4}})\\
&\geq {\rho}^-(s+l)\|h_{T_{01}}\|^2 -({\rho}^+(s+2l)-{\rho}^-(s+2l))\|h_{T_{01}}\|\|h_{T_{0}^c}\|_1/\sqrt{l}~~~~(\text{from Lemma~\ref{lem_T01}})\\
&\geq {\rho}^-(s+l)\|h_{T_{01}}\|^2 - ({\rho}^+(s+2l)-{\rho}^-(s+2l))\|h_{T_{01}}\|(3\|h_{T_0}\|_1/\sqrt{l}+4\|x^*_{T_0^c}\|_1/\sqrt{l}) \;\; (\text{from \eqref{eq:lem7.1}})\\
&\geq
\left({\rho}^-(s+l)-3\sqrt{s/l}({\rho}^+(s+2l)-{\rho}^-(s+2l))\right)\|h_{T_{01}}\|^2
-
\\
&\quad 4({\rho}^+(s+2l)-{\rho}^-(s+2l))\|x^*_{T_0^c}\|_1\|h_{T_{01}}\|/\sqrt{l}~~~~(\text{using $\|h_{T_0}\|_1\leq \sqrt{s}\|h_{T_0}\|\leq \sqrt{s}\|h_{T_{01}}\|$})\\
&\geq A_0(\Psi) \|h_{T_{01}}\|^2- A_1(\Psi)
\|h_{T_{01}}\|\|x^*_{T_0^c}\|_1 /\sqrt{l},
\end{align*}
which completes the proof.
\end{proof}
Similar claims or inequalities to~\eqref{eqn_lem4} can be found
in \citet{Candes07a, Candes08, Zhang09a}.

\begin{lemma} \label{lem_right.LASSO}
Assume that~\eqref{eqn_feasible_extend} holds. We have
\begin{subequations}
\begin{align}
  \|\tp h\|^2 \leq& {3\over 2}\lambda\Delta\|h\|_1 \leq
  6\sqrt{s}\lambda\Delta\|h_{T_{01}}\| + 6\lambda\Delta\|x^*_{T_0^c}\|_1,\label{eqn_lem_right:LASSO.1}\\
  \|\tp h\|^2 \leq& \tilde{M}\Delta^2.\label{eqn_lem_right:LASSO.2}
\end{align}
\end{subequations}
\end{lemma}
\begin{proof}
Denote the feasible region of~\eqref{eqn_ourformulation} as
\[
F:=\left\{x~|~\bar{\Phi}x-\bar{y}\geq 0,~\|\tp x-\tilde{y}\|_\infty\leq \Delta/2\right\}.
\]
Since $\hat{x}$ is the optimal solution
to~\eqref{eqn_ourformulation}, we have the optimality condition:
\begin{align*}
\tp^T(\tp\hat{x}-\tilde{y}) + \lambda\Delta \partial\|\hat{x}\|_1
\cap -N_F(\hat{x}) \neq \emptyset,
\end{align*}
where $N_F(\hat{x})$ denotes the normal cone of $F$ at the point
$\hat{x}$ and $\partial\|\hat{x}\|_1$ is the subgradient of the
function $\|.\|_1$ at the point $\hat{x}$. This condition is
equivalent to existence of $g\in \partial \|\hat{x}\|_1$ and $n\in
N_F(\hat{x})$ such that
$$\tp^T(\tp\hat{x}-\tilde{y})+\lambda\Delta g + n=0.$$
It follows that
\begin{align*}
& \quad \tp^T\tp h+ \tp^T(\tp{x^*}-\tilde{y})+\lambda\Delta g + n=0\\
&\Rightarrow h^T\tp^T\tp h+ h^T \tp^T(\tp x^* - \tilde{y}) +
\lambda\Delta h^Tg + h^Tn=0\\
&\Rightarrow \|\tp h\|^2 = -h^T \tp^T(\tp x^* -
\tilde{y}) - \lambda\Delta h^Tg - h^Tn \\
&\Rightarrow \|\tp h\|^2 \leq - h^T \tp^T(\tp x^* -
\tilde{y}) - \lambda\Delta h^Tg ~~~~(\text{using $x^*\in F$ and so $-h^Tn = (x^*-\hat{x})^Tn \leq 0$})\\
&\Rightarrow \|\tp h\|^2 \leq \|h\|_1\|\tp^T(\tp x^*
-\tilde{y})\|_\infty + \lambda\Delta\|h\|_1\|g\|_\infty.
\end{align*}
From $\|g\|_\infty\leq 1$ and~\eqref{eqn_feasible_extend}, we obtain
\begin{align*}
\|\tp h\|^2 &\leq \lambda\Delta\|h\|_1 / 2 + \lambda\Delta\|h\|_1\\
&= {3\over 2}\lambda\Delta \|h\|_1\\
&={3\over 2}\lambda\Delta (\|h_{T_0}\|_1 + \|h_{T_0^c}\|_1) \\
&\leq{3\over 2}\lambda\Delta (4\|h_{T_0}\|_1 + 4\|x^*_{T_0^c}\|_1)~~~~(\text{from~\eqref{eq:lem7.1}}) \\
&\leq 6\sqrt{s}\lambda\Delta \|h_{T_0}\| +
6\lambda\Delta\|x^*_{T_0^c}\|_1,
\end{align*}
which proves the first inequality.

From~\eqref{eq:lem1.1}, the second inequality is obtained by $ \|\tp
h\|^2 \leq \left( \sqrt{\tilde{M}\|\tp h\|_\infty}\right)^2\leq
\tilde{M}\Delta^2$.
\end{proof}

\subsection*{Proof of Theorem~\ref{thm_main_extend}}
\begin{proof}
First, assume that~\eqref{eqn_feasible} holds. Take $\Psi = \tp$ in
Lemma~\ref{lem_left.LASSO} and apply \eqref{eqn_lem_right:LASSO.1}.
We have
\[
A_0(\tp)\|h_{T_{01}}\|^2 -
({A_1({\tp})/\sqrt{l}})\|x^*_{T_{01}^c}\|_1\|h_{T_{01}}\| 
\leq  \|\tp h\|^2
\leq 6\sqrt{s}\lambda\Delta\|h_{T_{01}}\| +
6\lambda\Delta \|x^*_{T_0^c}\|_1.
\]
If follows that
\begin{align}
A_0(\tp)\|h_{T_{01}}\|^2 -
\left(({A_1(\tp)/\sqrt{l}})\|x^*_{T_{01}^c}\|_1+6\sqrt{s}\lambda\Delta \right)\|h_{T_{01}}\| \leq
6\lambda\Delta\|x^*_{T_0^c}\|_1. \label{eq:thm4.1}
\end{align}
Using $A_0(\tp) >0$ (which is assumed in the statement of the
theorem), we recall that for a quadratic inequality $ax^2 - bx \leq
c$ with $a,b,c>0$, one has
\begin{equation}
 x\leq {b+\sqrt{b^2+4ac} \over 2a} \leq
{2b+\sqrt{4ac}\over 2a} = {b\over a} + \sqrt{c\over a}.
\label{eqn_lem:LASSO.quatratic}
\end{equation}
Hence~\eqref{eq:thm4.1} implies that
\begin{align*}
\|h_{T_{01}}\| &\leq {1\over A_0(\tp)}\left(
{(A_1(\tp)/\sqrt{l})\|x^*_{T_{01}^c}\|_1} + 6\sqrt{s}\lambda\Delta\right) +
\sqrt{\frac{\lambda\Delta\|x^*_{T_0^c}\|_1}{A_0(\tp)}} \\
& = {6\sqrt{s}\lambda\Delta\over A_0(\tp)} + {A_1(\tp)\over
A_0(\tp)\sqrt{l}}\|x^*_{T_0^c}\|_1 + \sqrt{6\lambda\Delta\over A_0(\tp)}\|x^*_{T_0^c}\|_1^{1/2}.
\end{align*}
By invoking~\eqref{eq:lem7.2}, we prove \eqref{eqn_thm:LASSO.1_extend} by
\begin{align*}
\|h\|&\leq \sqrt{1+9s/l}\|h_{T_{01}}\| + \left(4/\sqrt{l} \right)\|x^*_{T_0^c}\|_1\\
&\leq {6 \sqrt{1+9s/l}\sqrt{s}\lambda\Delta\over A_0(\tp)} + \left(4+\frac{\sqrt{1+9s/l}A_1(\tp)}{A_0(\tp)}\right) \left(\|x^*_{T_0^c}\|_1/\sqrt{l} \right) +\sqrt{(1+9s/l)6\lambda\Delta\over A_0(\tp)}\|x^*_{T_0^c}\|_1^{1/2}\\
&= 6C_2(\tp)^2\sqrt{s}\lambda\Delta + C_1(\tp) \left( \|x^*_{T_0^c}\|_1/\sqrt{l} \right)+
2.5C_2(\tp)\sqrt{\lambda\Delta \|x^*_{T_0^c}\|_1}.
\end{align*}

Next we prove \eqref{eqn_thm:LASSO.2_extend}. Taking $\Psi = \tp$ in
Lemma~\ref{lem_left.LASSO} and applying
\eqref{eqn_lem_right:LASSO.2}, we have
\begin{align*}
A_0(\tp)\|h_{T_{01}}\|^2 -
{\left(A_1({\tp})/\sqrt{l} \right)}\|x^*_{T_{01}^c}\|_1\|h_{T_{01}}\| \leq \|\tp
h\|^2 \leq \tilde{M}\Delta^2.
\end{align*}
Using~\eqref{eqn_lem:LASSO.quatratic} again, one has
\[
 \|h_{T_{01}}\| \leq \frac{A_1(\tp)} {A_0(\tp)} \left( \|x^*_{T_0^c}\|_1/\sqrt{l} \right) + \frac{\sqrt{\tilde{M}}\Delta}{\sqrt{A_0(\tp)}}.
\]
By invoking~\eqref{eq:lem7.2}, we have
\begin{align*}
\|h\|\leq & \sqrt{1+9s/l}\|h_{T_{01}}\| +
\left(4/\sqrt{l} \right)\|x^*_{T_0^c}\|_1 \\
\leq &
\left(4+\frac{\sqrt{1+9s/l}A_1(\tp)}{A_0(\tp)}\right)\|x^*_{T_0^c}\|_1/\sqrt{l}
+ \sqrt{1+9s/l\over A_0(\tp)}\sqrt{\tilde{M}}\Delta,
\end{align*}
proving \eqref{eqn_thm:LASSO.2_extend}.

Note that all claims hold under the assumption
that~\eqref{eqn_feasible} is satisfied. Since
Lemma~\ref{lem:feasible} shows that~\eqref{eqn_feasible} holds with
probability at least $1-\pi$ with taking $\lambda =
\sqrt{2\log(2N/\pi)}\fmax$, we conclude that all claims hold with
the same probability.
\end{proof}

\subsection*{High-Probability Estimates of the Estimation Error}

For use in these results, we define the quantity
\begin{equation} \label{eq:chisat}
\chisat:={\bar{M}/M} = (M-\tilde{M})/M,
\end{equation}
which is the fraction of saturated measurements.
\begin{theorem}
  Assume $\Phi\in \mathbb{R}^{M\times N}$ to be a Gaussian random
  matrix, that is, each entry is i.i.d. and drawn from a standard
  Gaussian distribution $\mathcal{N}(0,1)$. Let
  $\tp\in\mathbb{R}^{\tilde{M}\times N}$ be the submatrix of $\Phi$
  taking $\tilde{M}$ rows from $\Phi$, with the remaining $\bar{M}$
  rows being used to form the other submatrix
  $\bar\Phi\in\mathbb{R}^{\bar{M}\times N}$, as defined in
  \eqref{eq:def.phis}.
  Then by choosing a threshold $\tau$ sufficiently small, and
  assuming that $\chisat$ satisfies the bound $\chisat(1-\log\chisat)
  \le \tau$, we have for any $k \ge 1$ such that $k \log N = o(M)$
  that, with probability larger than $1-O\left(\exp(-\Omega(M))\right)$, the following
  estimates hold:
\begin{subequations}
\label{eqn_thm3}
\begin{align}
\sqrt{{\rho}^+(k)}\leq& {17\over 16}\sqrt{M} + o\left(\sqrt{M}\right),\label{eqn_thm3.1}\\
\sqrt{{\rho}^-(k)}\geq& {15\over 16}\sqrt{M} - o\left(\sqrt{M}\right), \label{eqn_thm3.2}\\
\sqrt{\tilde{\rho}^+(k)}\leq & {17\over 16}\sqrt{\tilde{M}} + o\left(\sqrt{M}\right), \label{eqn_thm3.5}\\
\sqrt{\tilde{\rho}^-(k)}\geq & {15\over 16}\sqrt{\tilde{M}} -
o\left(\sqrt{M}\right). \label{eqn_thm3.6}
\end{align}
\end{subequations}
\label{thm_rhobound}
\end{theorem}
\begin{proof}
  From the definition of ${\rho}^+(k)$, we have
\[
\sqrt{{\rho}^+(k)}=\max_{|T|\leq k, T\subset \{1,2,...,
    N\}}\sigma_{\max}(\Phi_{T}),
\]
where $\sigma_{\max}(\Phi_T)$ is the maximal singular value of
$\Phi_T$. From \citet[Theorem~5.39]{Vershynin11}, we have for any
$t>0$ that
\[
\sigma_{\max}(\Phi_{T})\leq \sqrt{{M}}+O\left(\sqrt{k} \right)+t
\]
with probability larger than $1-O\left(\exp(-\Omega(t^2) \right)$. Since the
number of possible choices for $T$ is
\[
{N \choose k} \leq \left(\frac{eN}{k} \right)^k,
\label{eqn_thm3proof1}
\]
we have with probability at least
\[
1-{N \choose k} O \left(\exp{(-\Omega(t^2))} \right)\geq
1-O \left(\exp{(k\log{(eN/k)}}-\Omega(t^2) \right)
\]
that
\begin{align*}
\sqrt{{\rho}^+(k)} &= \max_{|T|\leq k, T\subset \{1,2,...,
N\}}\sigma_{\max}(\Phi_{T}) \leq  \sqrt{M}+O \left(\sqrt{k} \right)+t.
\end{align*}
Taking $t=\sqrt{M}/16$, and noting that $k = o(M)$, we obtain the
inequality~\eqref{eqn_thm3.1}, with probability at least
\begin{align*}
  & \quad 1-O(\exp{(k\log(eN/k)-\Omega(t^2))}\\
 &=1-O(\exp{(k\log(eN/k)-\Omega(M))}\\
 &=1-O(\exp{(o(M)-\Omega(M))}\\
  &\geq 1- O(\exp({-\Omega(M)}))
\end{align*}

The second inequality~\eqref{eqn_thm3.2} can be obtained similarly
from
\[
\min_{|T|\leq k, T\subset \{1,2,..., N\}}\sigma_{\min}(\Phi_{T})
\leq \sqrt{M}-O\left(\sqrt{k}\right)-t,
\]
where $\sigma_{\min}(\Phi_T)$ is the minimal singular value of
$\Phi_T$. (We set $t =\sqrt{M}/16$ as above.)

Next we prove~\eqref{eqn_thm3.5}. We have
\begin{align*}
\sqrt{\tilde{\rho}^+(k)} = &\max_{h, |T|\leq k}
\frac{\|\tilde{\Phi}_Th_T\|}{\|h_T\|} \leq \max_{|T|\leq k, |R|\leq
\tilde{M}}\sigma_{\max}(\Phi_{R,T}),
\end{align*}
where $R \subset \{1,2,\dotsc,M\}$ and $T\subset\{1,2,\dotsc,N\}$
are subsets of the row and column indices of $\Phi$, respectively,
and $\Phi_{R,T}$ is the submatrix of $\Phi$ consisting of rows in
$R$ and columns in $T$. We now apply the result in
\citet[Theorem~5.39]{Vershynin11} again: For any $t>0$, we have
\[
\sigma_{\max}(\Phi_{R,T})\leq \sqrt{{\tilde{M}}}+O\left(\sqrt{k}\right)+t
\]
with probability larger than $1-O(\exp{(-\Omega(t^2))})$. The number
of possible choices for $R$ is
\begin{align*}
{M\choose\bar{M}} \le \left( \frac{eM}{\bar{M}} \right)^{\bar{M}} &
= \left(\frac{e}{\chisat} \right)^{\chisat M} = \exp ( M \chisat
\log(e/\chisat) ) \le \exp (\tau M),
\end{align*}
so that the number of possible combinations for $(R,T)$ is bounded
as follows:
\[ {M\choose \bar{M}}{N \choose k} \le \exp \left(\tau
    M + k \log (eN/k)\right).
\]
We thus have
\begin{align*}
& \quad \mathbb{P}\left(\sqrt{\tilde{\rho}^+(k)} \leq \sqrt{{\tilde{M}}}+O\left(\sqrt{k}\right)+t\right)\\
&\geq \mathbb{P}\left(\max_{|R|\leq \tilde{M}, |T|\leq k}\sigma(\Phi_{R,T}) \leq \sqrt{{\tilde{M}}}+O\left(\sqrt{k}\right)+t \right)\\
&\geq 1-{M\choose \tilde{M}}{N \choose k} O(e^{-\Omega(t^2)})\\
&= 1-{M\choose \bar{M}}{N \choose k} O(e^{-\Omega(t^2)})\quad\text{(since $\bar{M}+\tilde{M} = M$)}\\
&= 1- O\left[\exp \left(\tau M + k\log ({eN}/{k})-\Omega(t^2)\right)
\right].
\end{align*}
Taking $t=\sqrt{\tilde{M}}/16$, and noting again that $k=o(M)$, we
obtain the inequality in \eqref{eqn_thm3.5}. Working further on the
probability bound, for this choice of $t$, we have
\begin{align*}
&\quad 1- O\left[\exp \left(\tau M + k\log ({eN}/{k})-\Omega(\tilde{M})\right) \right]\\
&= 1- O\left[\exp \left(\tau M + k\log ({eN}/{k})-\Omega(M)\right) \right]\\
&= 1 - O(\exp (-\Omega(M))),
\end{align*}
where the first equality follows from $\tilde{M}=(1-\chisat)M$ and
for the second equality we assume that $\tau$ is chosen small enough
to ensure that the $\Omega(M)$ term in the exponent dominates the
$\tau M$ term.

A similar procedure can be used to prove \eqref{eqn_thm3.6}.
\end{proof}

We conclude by deriving estimates of $\bar{C}_1(\tp)$,
$\bar{C}_2(\tp)$, and $\fmax$, that are used in the discussion at the
end of Section~\ref{theory}.

From Theorem~\ref{thm_rhobound}, we have that under assumptions (iii),
(iv), and (v), the quantity $A_1(\tp)$ defined in \eqref{eq:def.A1} is
bounded as follows:
\begin{align*}
\bar{A}_1(\tp) &= 4\left( \sqrt{\tilde{\rho}^+(3s)} +
\sqrt{\tilde{\rho}^-(3s)} \right)
 \left( \sqrt{\tilde{\rho}^+(3s)} - \sqrt{\tilde{\rho}^-(3s)} \right) \\
&\leq 4 \left( 2\sqrt{\tilde{M}}+o(\sqrt{M}) \right) \left({1\over 8}\sqrt{\tilde{M}}+o \left(\sqrt{M} \right) \right)\\
&= \tilde{M} + o(M) = \Omega(M).
\end{align*}
Using $s=l$, the quantity  $\bar{A}_0(\tp)$ defined in \eqref{eq:def.A0}
is bounded as follows:
\begin{align*}
\bar{A}_0(\tp)
&= \tilde{\rho}^- (2s) - \frac{3}{4} \bar{A}_1(\tilde\Phi)\\
&\ge {15\over 16}\tilde{M}-o(M) - {3\over 4}\tilde{M} - o(M)\\
&= {3\over 16}\tilde{M}-o(M)\\
&=  \Omega(M),
\end{align*}
for all sufficiently large dimensions and small saturation ratio
$\chisat$, since $\tilde{M} = (1-\chi) M$.  Using the estimates
above for $\bar{A}_0(\tp)$ and $\bar{A}_1(\tp)$,
in the definitions \eqref{eq:def.C1} and
\eqref{eq:def.C2}, we obtain
\[
\bar{C}_1(\tp)=4+\sqrt{10}{\bar{A}_1(\tp)}/{A_0(\tp)}=\Omega(1), \quad
\bar{C}_2(\tp)=\sqrt{10 / \bar{A}_0(\tp)}=\Omega \left(1/\sqrt{M} \right),
\]
as claimed. \sjwresolved{Had to fix some estimates here. {\rca I
have
    double checked it, but do not see any problems. Could you point it
    out?}  I fixed them already.}
Finally,  $\fmax$ can be estimated by
\[
\fmax=\sqrt{\tilde{\rho}^+(1)}\leq {17\over
16}\sqrt{\tilde{M}}+o(M) = O\left(\sqrt{M}\right).
\]



{
\bibliographystyle{model1b-num-names}
\bibliography{reference}
}




\end{document}